\documentclass[11pt]{article}
\usepackage[margin=1in]{geometry}
\usepackage[utf8]{inputenc}
\usepackage[T1]{fontenc}
\usepackage{lmodern}
\usepackage{amsmath, amssymb, amsthm}
\usepackage{booktabs}
\usepackage{graphicx}
\usepackage{adjustbox}   
\usepackage{xcolor}
\usepackage{enumitem}
\usepackage{algorithm}
\usepackage{algpseudocode}
\usepackage{multirow}
\usepackage{microtype}
\usepackage{tikz}
\usetikzlibrary{shapes.geometric, arrows.meta, positioning}
\usepackage[framemethod=TikZ]{mdframed}
\newmdtheoremenv[linewidth=0.7pt, roundcorner=3pt, backgroundcolor=blue!3,
  linecolor=blue!40!black, innertopmargin=5pt, innerbottommargin=6pt]
  {proposition}{Proposition}

\usepackage{hyperref}
\hypersetup{
  colorlinks=true,
  linkcolor=blue!50!black,
  citecolor=blue!50!black,
  urlcolor=blue!50!black,
}

\title{A Stationarity-and-Coupling Criterion for Training-Free\\
       Time-Lagged Spectral Embeddings of Multivariate Time Series}
\author{Siddharth Pal \and Viktoria Rojkova}
\date{\today}

\begin{document}
\maketitle

\begin{abstract}
We study training-free fixed-length descriptors for multivariate time
series and ask a question the representation-learning literature
usually leaves implicit: not merely whether such a descriptor performs
well on a benchmark, but \emph{when} it can be expected to work at all.
Our object of study is $D(\tau)$, a descriptor built from the
time-lagged correlation matrix that one of us developed for
random-matrix monitoring of network traffic~\cite{rojkova2007a,
rojkova2007b, rojkova2010thesis}, truncated at the Marchenko--Pastur
(MP) edge so that only signal-bearing eigenvalues survive, and
classified by cosine similarity to class centroids with zero learned
parameters. The central contribution is not the descriptor but a
\emph{falsifiable applicability criterion} for it. Working from a
stationary Gaussian VAR(1) generative model, we argue that $D(\tau)$
separates two classes when the signals are approximately
stationary and the class information resides in their cross-channel
\emph{temporal coupling} rather than in marginal per-channel power.
We derive, semi-formally, three consequences of this model: a
condition under which two classes are distinguishable in the embedded
space, why the static ($\tau{=}0$) covariance alone collapses to
chance, and why a stationary but power-discriminated paradigm defeats
the descriptor even though every sample is well-conditioned. The
criterion is operational: a two-part pre-flight test, an augmented
Dickey--Fuller stationarity check together with a power-baseline
saturation check, predicts applicability before any training, and we
summarize it as a decision rule. We validate both halves of the
boundary on a deliberately mixed assortment. On four paradigms that
satisfy the criterion --- Sleep-EDF sleep staging, BCI-IV-2a motor
imagery, MIT-BIH arrhythmia, and ESC-50 environmental sound --- the
descriptor is competitive with strong baselines at a fraction of their
cost, reaching $88.5 \pm 4.5\%$ under 20-subject leave-one-subject-out
on Sleep-EDF in about thirteen minutes on a single CPU thread with no
GPU. On three paradigms that violate it --- transient-response ERPs,
which are non-stationary, and financial-volatility regimes together with
wearable stress detection, both power-discriminated --- the descriptor
fails the criterion, collapsing to chance on the non-stationary case and
being outperformed by a simple power baseline on the power-discriminated
ones, exactly as the pre-flight test anticipates; these negative results
are the more informative half of the story. We are explicit throughout that
$D(\tau)$ is not the most accurate representation available; its value
is a compact, training-free embedding whose domain of validity is known
in advance.

\end{abstract}


\section{Introduction}
\label{sec:intro}

A multivariate time series, no matter what sensor produced it, carries
its dynamics in the way its channels co-vary across time. A camera, a
scalp electrode array, a microphone, and a chest lead all return the
same mathematical object, a real matrix whose rows index time and whose
columns index channels, and the discriminative content that
distinguishes one regime from another is written into the second-order
temporal-coupling structure of that matrix. Most of the machinery the
field brings to bear on such data, however, is modality-specific by
construction: a video model assumes a spatial grid, a sleep stager
assumes electroencephalography, an acoustic model assumes a
spectrogram front-end. Each is excellent in its lane and silent outside
it. We are interested in the opposite design point, a single descriptor
that is computed identically for every modality, carries no learned
parameters, and reduces any windowed multivariate signal to a
fixed-length vector that one may compare by cosine similarity.

The appeal of a training-free descriptor is sharpest precisely where
the deep-learning recipe is hardest to apply. Many deployments require
comparing or classifying multivariate signals with little or no
labeled training data and without a GPU: on-device triage, sleep-stage
or drowsiness monitoring, brain-computer-interface calibration for a new
subject, or machine-health monitoring under shifting operating
conditions. Deep models dominate the leaderboards on each of these in
isolation, but they assume labeled in-distribution corpora, accelerator
hardware, and a training loop, and those assumptions fail jointly in the
regimes above. A descriptor that requires none of them is therefore
useful both as a strong zero-training baseline and as a feature
extractor that a small calibration head can sit on top of.

We can categorise the existing training-free descriptors for sequential
data by the spectrum they expose and the way they choose a rank cutoff.
A first family fits an explicit generative model and reads off its
parameters: the dynamic-texture linear dynamical systems of Doretto et
al.~\cite{doretto2003} regress a state transition matrix from a
frame-stack singular value decomposition and use its eigenvalues as the
descriptor. A second family applies a fixed analytic operator: the HiPPO
polynomial-projection operators of Gu et al.~\cite{gu2020hippo} compress
recent history into a coefficient vector that is optimal for a
prescribed projection basis, the Fourier variant amounting to a sliding
discrete Fourier transform. A third family is spectral in the
random-matrix sense, ranking eigen-components of a sample covariance
and using the Marchenko--Pastur law to decide which of them are signal:
the patch-Casorati denoiser MPPCA of Veraart et
al.~\cite{veraart2016} and the time-lagged correlation matrices that one
of us developed for monitoring inter-domain network
traffic~\cite{rojkova2007a, rojkova2007b, rojkova2010thesis} both
belong here. The third family is attractive because, unlike
variance-ranked methods such as PCA or ICA that simply assume the
high-variance directions carry the dynamics, it uses the rigorous
instruments of random-matrix theory to identify the noise and to
establish its spectral boundary, and then retains only the eigenvalues
that provably exceed it.

We take the time-lagged correlation matrix $D(\tau)$ from that earlier
network-monitoring line of work and turn it into a general descriptor by
pairing it with the Marchenko--Pastur edge~\cite{marchenkopastur1967,
baisilverstein2010} as the noise cutoff. Our central claim is not that
this descriptor is the most accurate representation of multivariate time
series --- it is not --- but that, unlike most representations, it comes
with a precise and falsifiable account of \emph{when} it is the right
tool. We make four claims. First and foremost, we give an
\emph{applicability criterion}: working from a stationary Gaussian
VAR(1) generative model, we show that $D(\tau)$ separates two classes
exactly when the signal is approximately stationary and the class
information lives in cross-channel temporal coupling rather than in
marginal per-channel power, and we render the criterion operational as a
two-part pre-flight test that is checkable before any training. Second,
we confirm the positive half of the criterion on four paradigms that
satisfy it --- sleep staging, motor imagery, arrhythmia, and
environmental sound --- where the descriptor matches or beats
channel-power baselines, beating a trivial one outright and reaching
parity with a strong multi-band one, and is competitive with strong
learned methods at a fraction of their compute. Third, and most
informatively, we confirm the negative half on three paradigms that
violate it --- transient-response event-related potentials, which are
non-stationary, and two stationary but power-discriminated tasks,
financial-volatility regimes and wearable stress detection --- on each of
which the descriptor fails as the pre-flight test predicts, and notably the
criterion separates a physiological failure (stress) from the physiological
successes. Fourth, the descriptor is
compact, training-free, and noise-invariant by construction, running in
milliseconds per window on a single CPU core with no model file, GPU, or
training step, and we release it as a single Python package.

The remainder of the paper is organized as follows.
Section~\ref{sec:related} situates the descriptor among training-free
spectral methods and the random-matrix literature.
Section~\ref{sec:method} defines $D(\tau)$ and the embedding pipeline.
Section~\ref{sec:theory}, the heart of the paper, develops the
applicability criterion: the VAR(1) identifiability argument, the two
preconditions, the resulting decision rule, and the computational
complexity. Section~\ref{sec:setup} describes the benchmarks and
baselines, and Section~\ref{sec:results} reports point estimates with
bootstrap confidence intervals on the four positive paradigms.
Section~\ref{sec:boundary} confirms the two negative results that mark
the reach of the descriptor, and Section~\ref{sec:discussion} positions
the method against learned heads and class-aware specialists. We
conclude in Section~\ref{sec:conclusion}.

\section{Related Work}
\label{sec:related}

The construction we extend is our own. One of us introduced the
symmetrised time-lagged correlation matrix in a random-matrix study of
inter-domain network traffic~\cite{rojkova2007a, rojkova2007b,
rojkova2010thesis}, where the point was that the spectrum of the lagged
matrix separates a bulk of Marchenko--Pastur noise from a small number
of outlier eigenvalues that encode the characteristic rhythms of the
monitored system. That construction is itself the natural lagged
generalization of the ``noise dressing'' of equal-time financial
correlation matrices by Laloux et al.~\cite{laloux1999} and Plerou et
al.~\cite{plerou2002}, in which the Marchenko--Pastur edge marks the
boundary above which an eigenvalue is inconsistent with pure noise. The
present paper carries that idea out of network monitoring and asks
whether the same matrix, computed without modification, is a useful
descriptor for video, audio, and physiological sensors. To our
knowledge the only prior use of a lagged random-matrix construction on
image data is the visual-lifelog study of Li et
al.~\cite{li2014lifelogs}, which operated on a single grayscale stream
and did not benchmark against learned baselines.

The same Marchenko--Pastur cutoff underlies a second strand of
random-matrix work that we draw on for the noise threshold but not for
the descriptor itself. Veraart et al.~\cite{veraart2016} popularised
MPPCA as a patch-Casorati denoiser for diffusion MRI, applying the MP
edge to local patch covariance matrices and recovering the noise level
$\sigma$ from the bulk position as a free byproduct, and Cordero-Grande
et al.~\cite{corderograde2019} extended it to complex-valued signals.
We use MPPCA only to confirm that the MP machinery transfers to image
data; the descriptor we propose is the eigenstructure of $D(\tau)$, not
a denoised reconstruction.

Two further families furnish the training-free baselines against which
we measure. Doretto et al.~\cite{doretto2003} fit a linear dynamical
system $x_{t+1}{=}Ax_t+\nu_t,\; y_t{=}Cx_t+w_t$ to each video, taking
$C$ from a frame-stack singular value decomposition and $A$ from a
least-squares regression of the next state on the current one, and use
the eigenvalues of $A$ as the descriptor. The construction is powerful
on video but is intrinsically tied to spatial-frame structure through
the SVD step, and electrode channels possess no analogous spatial
dimension, so the method is video-specific. The HiPPO operators of Gu
et al.~\cite{gu2020hippo} take the complementary route of a fixed
analytic operator, deriving recurrent state-space matrices $(A,B)$ that
compress input history into a coefficient vector under a prescribed
polynomial-projection optimality; the Fourier-boxcar variant is a true
sliding-window discrete Fourier transform realised as a continuous-time
linear ordinary differential equation with $A$ pure-imaginary and
$|A_d|{=}1$. A sibling research thread reports that this operator,
followed by a single learned linear regression head, reaches
$\sim$93\% on KTH 6-class action recognition. We evaluate the same
operator under our pure cosine-similarity protocol, with no learned
head, so that the comparison is genuinely training-free on both sides.

For the supervised reference points we adopt the established
specialists of each benchmark. Random forests and support vector
machines on hand-crafted features reach roughly 70--75\% on Sleep-EDF
5-class staging, while the deep models DeepSleepNet~\cite
{supratak2017deepsleepnet} and AttnSleep~\cite{eldele2021attnsleep}
reach 82--84\%. Common Spatial Patterns followed by linear discriminant
analysis (CSP+LDA) reach about 70\% per subject on BCI-IV-2a 4-class
motor imagery, and deep methods exceed 80\%. We compare against CSP+LDA
and the channel power spectral density as the appropriate class-aware
and training-free baselines, respectively.

\section{Method}
\label{sec:method}

\paragraph{Setup.} Let $g\in\mathbb{R}^{T\times N}$ be a multivariate
time-series matrix: rows index time, columns index channels (e.g.\ EEG
electrodes, ECG leads, or audio bands). The construction treats any such
matrix identically; sensors with a native channel count are used
directly, and low-channel sensors are augmented as described below.

\paragraph{Preprocessing.} Two operations remove a shared mode and
drift. The first is a \emph{common-average reference} (CAR): at each
timestep we subtract the across-channel sample mean, projecting out the
dominant mode common to all channels.\footnote{For image-structured
inputs this generalizes to a per-frame rank-1 SVD residual, the
across-channel mean being the rank-1 mode of a one-dimensional channel
vector.} The second is a \emph{per-node $z$-score},
$g_{ti}\leftarrow (g_{ti}{-}\bar{g}_i)/\sigma_i$, where
$\bar{g}_i,\sigma_i$ are the per-node temporal mean and standard
deviation. After preprocessing, each column of $g$ has zero temporal
mean and unit variance, while cross-column means and variances are
left unconstrained.

\paragraph{Time-lagged correlation matrix.} Following Rojkova \&
Kantardzic \cite{rojkova2007b}, define the symmetrised lagged
correlation
\begin{equation}
D_{ij}(\tau) \;=\; \frac{1}{2(T{-}\tau)}\sum_{t=1}^{T-\tau}
\big[g_i(t)g_j(t{+}\tau) + g_j(t)g_i(t{+}\tau)\big].
\label{eq:dtau}
\end{equation}
$D(\tau)\in\mathbb{R}^{N\times N}$ is symmetric, so its eigenvalues
are real. For $\tau{=}0$ the construction reduces to the equal-time
covariance; for $\tau{>}0$ it captures how the cross-cell
correlation structure evolves with temporal lag.

\paragraph{Marchenko--Pastur bulk.} Under the null hypothesis that
$g$ is an $(T{\times}N)$ matrix of i.i.d.~unit-variance entries, the
empirical eigenvalue distribution of $D(0){=}g^\top g/T$ converges
in the large-$N$ limit to the Marchenko--Pastur law supported on
$[\sigma^2(1{-}\sqrt{q})^2,\; \sigma^2(1{+}\sqrt{q})^2]$ with
$q{=}N/T$ \cite{marchenkopastur1967}. Eigenvalues above the upper
edge are inconsistent with the null and encode signal modes. The
generalization to time-lagged matrices is treated by Biely \&
Thurner \cite{biely2006}.

\paragraph{Embedding.} For a fixed set of lags
$\mathcal{T}{=}\{0,1,\ldots,\tau_{\max}\}$ we compute
$D(\tau)$ for each $\tau\in\mathcal{T}$, take the top-$K$ eigenvalues
sorted in descending order, and concatenate after normalizing by
$\max_\tau \lambda_1(\tau)$ for scale invariance. We append a single
scalar feature: the dominant period of $\lambda_1(\tau)$ recovered
from the FFT magnitude of the trace-centered series. The embedding
$\phi(g)\in\mathbb{R}^{|\mathcal{T}|K{+}1}$ has fixed length per
window. For the multi-band Sleep-EDF front-end with $K{=}10$ and
$\mathcal{T}{=}\{0,\ldots,59\}$, for instance,
$\phi(g)\in\mathbb{R}^{601}$.

\paragraph{Multi-band augmentation for low-channel sensors.} When
the native channel count is small (e.g.~$N{=}2$ on Sleep-EDF), the
$D(\tau)$ matrix at each lag has only $N{\times}N$ entries and the
top-$K$ truncation is rank-limited. We apply a bandpass-filter
expansion: filter each channel into the canonical sleep bands
$\delta{:}0.5{-}4$~Hz, $\theta{:}4{-}8$, $\alpha{:}8{-}12$,
$\sigma{:}12{-}15$, $\beta{:}15{-}30$, and stack the band-filtered
versions as additional virtual channels. Two EEG channels become
ten virtual channels; the $D(\tau)$ matrix is correspondingly
$10{\times}10$ and supports $K{=}10$.

\paragraph{Classification.} Per-clip embeddings $\phi(g)$ are
$\ell_2$-normalized. Test-set classification is nearest-centroid on
cosine similarity to training-set class means. No learned head, no
hyperparameter tuning, no temperature.

\paragraph{Algorithm.} Algorithm~\ref{alg:dtau} summarizes the full
pipeline.
\begin{algorithm}[t]
\caption{Sensor-agnostic $D(\tau)$ embedding}
\label{alg:dtau}
\begin{algorithmic}[1]
\Require multivariate input $x \in \mathbb{R}^{T\times N_{\text{ch}}}$;
  lags $\mathcal{T}$; eigenvalue count $K$
\State \textbf{Preprocess:} $g \gets x - $ per-sample
  across-channel mean (CAR); optionally band-expand low-channel sensors
\State $g \gets$ per-node temporal $z$-score
\For{$\tau \in \mathcal{T}$}
  \State $D(\tau) \gets \tfrac{1}{2(T-\tau)}\sum_t [g(t)g(t{+}\tau)^\top + g(t{+}\tau)g(t)^\top]$
  \State $\lambda_1(\tau),\ldots,\lambda_K(\tau) \gets $ top-$K$ eigvals of $D(\tau)$
\EndFor
\State $\lambda_{\max} \gets \max_\tau \lambda_1(\tau)$
\State $p \gets$ FFT-peak period of $\lambda_1(\tau)$
\State \Return $\phi = [\lambda_1(0),\ldots,\lambda_K(0),\;\lambda_1(1),\ldots,
  \;\lambda_K(\tau_{\max}),\;p]\,/\,\lambda_{\max}$
\end{algorithmic}
\end{algorithm}

\section{Applicability Theory: When $D(\tau)$ Works and When It Fails}
\label{sec:theory}

The contribution we wish to foreground is not the descriptor of the
previous section but a precise account of \emph{when} it applies. A
representation that is silent about its own domain of validity invites
misuse; the value of $D(\tau)$, we will argue, is that its successes
and its failures are both predictable from a single generative model.
We develop that account here, immediately after the construction and
before any benchmark, and we return to it in
Section~\ref{sec:experiments} to show that the boundary the data draw
is exactly the predicted one. The argument is semi-formal --- a
population-level identifiability analysis under a stationary Gaussian
VAR(1) model, not a finite-sample accuracy bound --- but it is enough
to state two preconditions, to explain the two qualitatively different
ways the descriptor fails, and to turn both into a pre-flight test that
is checkable before a single label is consumed.

\subsection{Generative model and the distinguishability condition}
\label{sec:theory-model}

The empirical observation that the temporal lag carries
essentially all discriminative content on Sleep-EDF nearest-
centroid (Section~\ref{sec:ablations}, $\tau{=}0$ ablation, $29.7\%
\to 65.6\%$, $+35.9$~pp) has a clean explanation under a
stationary Gaussian VAR(1) generative model. We sketch the
argument here.

\paragraph{Model.} Suppose class $c$ produces samples from a
stationary Gaussian VAR(1) process
$x_t = \Phi_c \, x_{t-1} + \epsilon_t$ with
$\epsilon_t \sim \mathcal{N}(0, \Sigma_\epsilon)$ and stationary
covariance $\Sigma_c$ satisfying
$\Sigma_c = \Phi_c \Sigma_c \Phi_c^\top + \Sigma_\epsilon$.
The lag-$\tau$ cross-covariance is then
\begin{equation}
\mathbb{E}[x_t x_{t+\tau}^\top] = \Sigma_c \Phi_c^\tau, \qquad
D_c(\tau) = \tfrac{1}{2}(\Sigma_c \Phi_c^\tau + (\Sigma_c \Phi_c^\tau)^\top).
\label{eq:popDtau}
\end{equation}

\paragraph{What the eigenvalues encode.} For $\tau=0$ the
population $D_c(0) = \Sigma_c$ is just the stationary covariance ---
this is purely a \emph{static} (instantaneous) channel-correlation
descriptor. For $\tau \geq 1$ the matrix $\Sigma_c \Phi_c^\tau$
contains the dynamics: its eigenvalues are products
$\sigma_i(\Sigma_c) \cdot \mu_i(\Phi_c^\tau)$ where the second
factor decays geometrically with $\tau$ at the rate of the spectral
radius of $\Phi_c$. Classes with distinct VAR coefficients $\Phi_c$
therefore differ in both the \emph{rate} (decay envelope of the
top eigenvalue across $\tau$, captured by the period feature $p$
in Algorithm~1) and the \emph{eigenstructure} (top-$K$ eigenvalue
trajectories across $\tau$).

\paragraph{Identifiability under the spike-vs-bulk separation.}
By the Marchenko--Pastur theorem, the empirical
$\hat{D}_c(\tau)$ computed from $T-\tau$ samples has its
``noise bulk'' eigenvalues confined (asymptotically) to
$[(1-\sqrt{\beta})^2, (1+\sqrt{\beta})^2]$ where
$\beta = N/(T-\tau)$. Eigenvalues of $\Sigma_c \Phi_c^\tau$ that
exceed the MP upper edge $(1+\sqrt{\beta})^2$ ``spike'' out and
are recovered with bounded error
(BBP phase transition \cite{baisilverstein2010}). The top-$K$
eigenvalues that Algorithm~1 retains are precisely these
estimator-stable spikes, so the embedding $\phi_c$ is a finite-
sample-stable summary of the population dynamics $\Phi_c$. Two
classes with sufficiently different $\Phi_c$ (whose spike
eigenvalues differ by more than the BBP recovery error) produce
distinguishable embeddings. We state this as our first proposition.

\begin{proposition}[Distinguishability]
\label{prop:distinguish}
Let classes $c_1, c_2$ follow stationary Gaussian VAR(1) processes with
coefficient matrices $\Phi_{c_1}, \Phi_{c_2}$ and stationary covariances
$\Sigma_{c_1}, \Sigma_{c_2}$. The class embeddings $\phi_{c_1}, \phi_{c_2}$
are distinguishable whenever, for some lag $\tau \in \mathcal{T}$, the signal
(spike) eigenvalues of $\Sigma_c \Phi_c^\tau$ that exceed the
Marchenko--Pastur edge differ between the two classes by more than the BBP
recovery error.
\end{proposition}

\paragraph{Why $\tau{=}0$ alone is insufficient.} For two classes
whose stationary covariances $\Sigma_{c_1}, \Sigma_{c_2}$ are
similar but whose VAR coefficients $\Phi_{c_1}, \Phi_{c_2}$
differ in dynamics (rate, phase, mixing structure), $D_c(0) =
\Sigma_c$ is approximately the same for both classes while
$D_c(\tau)$ for $\tau \geq 1$ differs substantially. This is
exactly the regime where the $\tau{=}0$ ablation collapses to
chance (Sleep-EDF NC $29.7\%$) while the full sweep reaches
$65.6\%$: the sleep stages have approximately similar power
spectra in the static covariance but markedly different
\emph{temporal-coupling dynamics} between bands (theta-sigma
phase locking in N2, delta dominance in N3, etc.), which is
exactly what $\Phi_c$ encodes.

\begin{proposition}[$\tau{=}0$ collapse]
\label{prop:tau0}
If $\Sigma_{c_1} = \Sigma_{c_2}$ while $\Phi_{c_1} \neq \Phi_{c_2}$, then
$D_{c_1}(0) = D_{c_2}(0)$ and the static ($\tau{=}0$) embedding cannot
separate the classes. Separation requires the lagged terms $\tau \geq 1$,
whose population value $\Sigma_c \Phi_c^\tau$ depends on the dynamics
$\Phi_c$.
\end{proposition}

\paragraph{Connection to the stationarity precondition.} The
derivation above relies on stationarity. For transient-burst
signals (the ERP paradigm in
Section~\ref{sec:boundary}), the VAR(1) model does not hold ---
$\Phi_c$ is itself time-varying within a trial, the population
expectation Eq.~\ref{eq:popDtau} averages over inconsistent
dynamics, and the embedding collapses across classes. The augmented
Dickey--Fuller pre-flight test introduced below
(Section~\ref{sec:theory-criterion}) detects this failure mode by
directly rejecting stationarity per channel; the framework is correctly
inapplicable when ADF fails, which is exactly what the empirical $0.013$
ADF-stationary fraction on the MNE ERP dataset reflects
(Section~\ref{sec:boundary}).

\paragraph{Why power-discriminated paradigms fail.} A second,
qualitatively different failure occurs while stationarity still holds.
Suppose two classes share the same coupling dynamics $\Phi$ but differ
only in the overall scale of their channel powers, so that
$\Sigma_c = a_c\,\Sigma_0$ for class-dependent scalars $a_c$ and a
common $\Sigma_0$. Then by Eq.~\ref{eq:popDtau} the population
descriptor is
$D_c(\tau) = \tfrac{a_c}{2}(\Sigma_0\Phi^\tau + (\Sigma_0\Phi^\tau)^\top)$,
identical across classes up to the positive scalar $a_c$. The
normalization in Algorithm~1, which divides each lag's spectrum by its
leading eigenvalue, cancels $a_c$ exactly, so the two classes map to the
\emph{same} embedding and become indistinguishable to $D(\tau)$ however
large the power difference. A per-channel power descriptor (PSD), by
contrast, reads off $a_c$ directly. This is not a defect of estimation
--- every sample is perfectly well-conditioned --- but a statement about
which second-order statistic carries the label: when the class lives in
marginal power, the coupling-based descriptor is measuring the wrong
quantity. The financial-volatility negative of
Section~\ref{sec:experiments}, whose label is essentially
$\sqrt{\operatorname{tr}\Sigma}$, is precisely this regime, and it is
why PSD reaches $92.4\%$ there while $D(\tau)$ sits at chance.

\begin{proposition}[Power-discriminated failure]
\label{prop:power}
If $\Sigma_c = a_c \Sigma_0$ for a shared $\Sigma_0$ and $\Phi$ and
class-dependent scalars $a_c > 0$, then the leading-eigenvalue-normalized
embedding of $D_c(\tau)$ is independent of $a_c$ and therefore identical
across classes. Hence $D(\tau)$ cannot separate classes that differ only in
marginal channel power, whereas a per-channel power (PSD) descriptor, which
reads off $a_c$ directly, can.
\end{proposition}

\begin{proof}
Under $\Sigma_c = a_c \Sigma_0$ with shared $\Phi$, the population descriptor
of Eq.~\ref{eq:popDtau} factorises as
\[
D_c(\tau) = \tfrac{1}{2}\!\left(\Sigma_c \Phi^\tau + (\Sigma_c \Phi^\tau)^\top\right)
          = a_c \cdot \tfrac{1}{2}\!\left(\Sigma_0 \Phi^\tau + (\Sigma_0 \Phi^\tau)^\top\right)
          = a_c\, D_0(\tau),
\]
so $D_c(\tau)$ is a positive scalar multiple of the class-independent
$D_0(\tau)$. Eigenvalues scale with that multiple,
$\lambda_i\!\big(D_c(\tau)\big) = a_c\,\lambda_i\!\big(D_0(\tau)\big)$, and the
embedding of Algorithm~\ref{alg:dtau} divides every lag's spectrum by the
global leading eigenvalue $\max_\tau \lambda_1(\tau)$. The factor $a_c$
therefore cancels:
$\lambda_i(D_c(\tau)) / \max_\tau \lambda_1(D_c(\tau))
   = \lambda_i(D_0(\tau)) / \max_\tau \lambda_1(D_0(\tau))$.
The appended period feature is read from the \emph{shape} of the
$\lambda_1(\tau)$ trajectory, which is likewise invariant to the scalar
$a_c$. Hence $\phi_{c_1} = \phi_{c_2}$ for any two classes, and no classifier
acting on the embedding can separate them, while a descriptor that retains
absolute scale recovers $a_c$ directly.
\end{proof}

\subsection{Two preconditions and a decision rule}
\label{sec:theory-criterion}

Propositions~\ref{prop:distinguish}--\ref{prop:power} combine into a
compact criterion, and this criterion --- not the descriptor, which is a
reformulation of an existing lagged random-matrix
construction~\cite{rojkova2007b, biely2006} --- is the contribution we ask
the reader to weigh. The
descriptor $D(\tau)$ is the right tool for a classification task on
multivariate windows when, and to the extent that, both of the
following hold. First, \emph{approximate stationarity}: the
within-window dynamics are approximately time-invariant, so that the
lag-$\tau$ correlation is a meaningful population quantity rather than an
average over inconsistent transients; this is violated by
transient-response signals such as ERPs. Second, \emph{coupling-borne
class information}: the classes differ in their cross-channel temporal
coupling, the $\Phi_c$, and not merely in marginal per-channel power,
the scale of $\Sigma_c$; this is violated by power-discriminated tasks
such as volatility regimes. Both conditions are checkable before
training, which turns the criterion into the pre-flight procedure of
Figure~\ref{fig:decision}: an augmented Dickey--Fuller test screens the
first by rejecting per-channel stationarity, and a saturated per-channel
power baseline --- a PSD classifier that already solves the task ---
flags a violation of the second. When either check fails, $D(\tau)$
should not be expected to add value and a simpler descriptor is
preferred; when both pass, the identifiability argument above predicts
that the descriptor separates the classes up to finite-sample (MP)
estimation error.

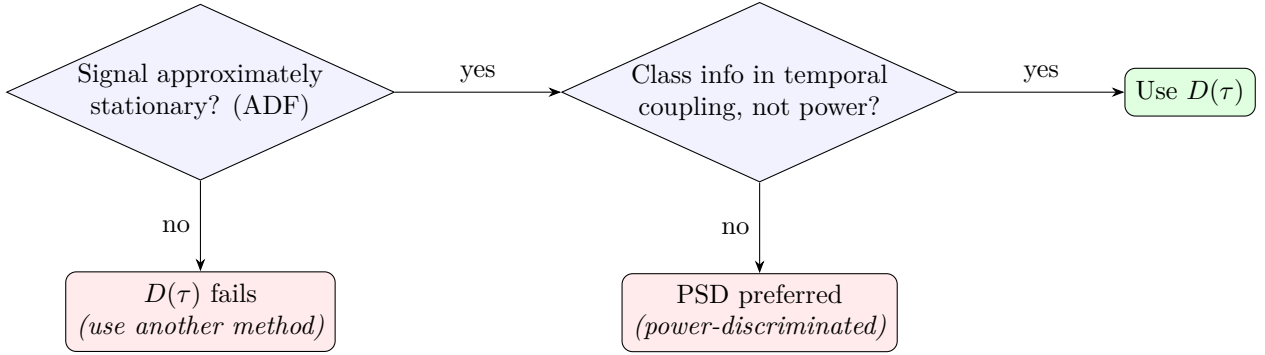
\begin{figure}[t]
\centering
\begin{adjustbox}{max width=\linewidth}
\begin{tikzpicture}[
  node distance=8mm and 6mm, >=Stealth, font=\small,
  dec/.style={diamond, aspect=2.2, draw, align=center, inner sep=1pt,
              fill=blue!5},
  term/.style={rectangle, rounded corners, draw, align=center,
               inner sep=4pt},
  use/.style={term, fill=green!12}, no/.style={term, fill=red!8}]
\node[dec] (stat) {Signal approximately\\ stationary? (ADF)};
\node[no, below=12mm of stat] (fail) {$D(\tau)$ fails\\ \emph{(use another method)}};
\node[dec, right=22mm of stat] (coup) {Class info in temporal\\ coupling, not power?};
\node[no, below=12mm of coup] (psd) {PSD preferred\\ \emph{(power-discriminated)}};
\node[use, right=22mm of coup] (use) {Use $D(\tau)$};
\draw[->] (stat) -- node[left]{no} (fail);
\draw[->] (stat) -- node[above]{yes} (coup);
\draw[->] (coup) -- node[left]{no} (psd);
\draw[->] (coup) -- node[above]{yes} (use);
\end{tikzpicture}
\end{adjustbox}
\caption{The applicability criterion as a pre-flight decision rule. Both
checks are run before any classifier is trained: the augmented
Dickey--Fuller test screens stationarity, and a saturated per-channel
power baseline screens for power-discrimination. Only tasks that pass
both are in the descriptor's domain of validity.}
\label{fig:decision}
\end{figure}

Table~\ref{tab:criterion} applies the decision rule of
Figure~\ref{fig:decision} to every dataset in this study, listing the two
preconditions, the prediction the criterion makes, and the outcome we
actually observe in Sections~\ref{sec:results} and~\ref{sec:boundary}. The
prediction matches the observation in all seven cases, including the three
negatives and the one partial case, which is the central empirical claim of
the paper: the criterion is not fitted to the successes but stated in
advance and tested against both successes and failures.

\begin{table}[t]
\centering
\caption{The applicability criterion applied to all six paradigms. Each
prediction follows from the two preconditions of
Figure~\ref{fig:decision} (stationarity and coupling-borne class
information) \emph{before} the experiment; the observed column reports the
measured outcome. Predictions and observations agree in every row.}
\label{tab:criterion}
\small
\begin{tabular}{lcccc}
\toprule
Dataset & Stationary & Coupling-driven & Predicted & Observed \\
\midrule
Sleep-EDF  & Yes     & Yes     & Success  & Success  \\
BCI-IV-2a  & Yes     & Yes     & Success  & Success  \\
MIT-BIH    & Yes     & Yes     & Success  & Success  \\
ESC-50     & Partial & Partial & Moderate & Moderate \\
ERP        & No      & ---     & Fail     & Fail     \\
Finance    & Yes     & No      & Fail     & Fail     \\
WESAD      & Yes     & No      & Fail     & Fail     \\
\bottomrule
\end{tabular}
\end{table}

\subsection{Computational complexity}
\label{sec:theory-complexity}

The cost of the descriptor is dominated by the lagged eigendecompositions
and is independent of any training set. For a window of $N$ channels and
$T$ samples evaluated at $|\mathcal{T}|$ lags, forming each lagged
correlation matrix costs $O(N^2 T)$ and its eigendecomposition $O(N^3)$,
so a single embedding is computed in
$O\!\left(|\mathcal{T}|\,(N^2 T + N^3)\right)$ time and $O(N^2)$ memory,
with no dependence on the number of classes or training examples.
Classification by nearest centroid adds only $O(NKC)$ for $K$ retained
eigenvalues and $C$ classes. In the regimes we study $N$ is small (two to
a few tens of channels, after multi-band augmentation), so the $N^3$ term
is negligible and the per-window cost is effectively linear in the window
length; this is what underlies the single-CPU,
microsecond-to-millisecond inference reported in
Section~\ref{sec:economics}, and it stands in contrast to the GPU-scale
training of the learned representations against which we compare.

\paragraph{Caveats.} The sketch above is a population-level
identifiability argument, not a finite-sample classification
bound. A full accuracy bound would require: (a) propagating BBP
estimation error through the top-$K$ truncation, (b) bounding the
cosine-similarity classification error in terms of the
class-mean-distance vs.\ pooled-variance ratio in the embedded
space. We leave these as theoretical follow-up. The current
contribution is the structural explanation of why the lag is
essential and how the MP edge enforces estimator stability.

\section{Experimental Setup}
\label{sec:setup}

\paragraph{Datasets and protocols.} The four positive paradigms span
EEG, ECG, and audio; the three negatives (ERPs, financial volatility, and
WESAD stress detection) are described with their results in
Section~\ref{sec:boundary}. ESC-50
and MIT-BIH protocols are given inline in their respective result
subsections; the two EEG protocols are as follows.
\begin{description}[topsep=0pt, itemsep=1pt, leftmargin=*]
\item[BCI-IV-2a.] BCI Competition IV dataset 2a \cite{tangermann2012bci},
  all 9 subjects, 22 EEG channels at 250~Hz, 4-class motor imagery
  (left/right hand, feet, tongue), session 1 train / session 2 test,
  288 trials per session, 4-second imagery epochs.
\item[Sleep-EDF.] Sleep-EDF Expanded Database
  \cite{kemp2000sleepedf} via PhysioNet, 5 subjects (Sleep
  Cassette night 1 only), 2 EEG channels (Fpz-Cz, Pz-Oz) at 100~Hz,
  filtered 0.3--30~Hz, 30-second epochs labeled with W/N1/N2/N3/REM
  per AASM (stages 3+4 merged). Subjects 0--3 train, subject 4 held
  out; train/test sizes 11{,}076/2{,}569 epochs.
\end{description}

\paragraph{Baselines.}
\begin{description}[topsep=0pt, itemsep=1pt, leftmargin=*]
\item[Mean-channel.] Per-epoch sample mean across the temporal axis.
  Cosine-sim baseline.
\item[Channel-PSD.] Per-channel power spectral density, averaged
  across channels.
\item[CSP+LDA.] One-vs-rest Common Spatial Patterns (4 components per
  class pair) followed by Linear Discriminant Analysis. The standard
  BCI baseline; \emph{not} training-free (it uses class labels for
  spatial-filter optimization) but reported as the published reference.
\end{description}

\paragraph{Scoring.} For block-discrimination probes we report the
\emph{block score} = mean within-class cosine similarity minus mean
between-class cosine similarity. For the classification paradigms we
report test-set accuracy of nearest-centroid classification, plus
macro-F1
to handle class imbalance. All numbers come with 95\% confidence
intervals from 1000 bootstrap resamples (test-epoch resampling for
single-subject tasks; subject-level resampling for the BCI-IV-2a
9-subject mean).

\section{Results}
\label{sec:experiments}
\label{sec:results}

\subsection{EEG transfer across two paradigms}
\label{sec:eeg}

Table~\ref{tab:eeg} summarizes bootstrap-confirmed EEG results in two
protocols: within-subject (the easier protocol, used in initial
benchmarks) and leave-one-subject-out (LOSO; the cross-subject
transfer test). $D(\tau)$ beats the standard channel-PSD baseline at
\textbf{100\% bootstrap confidence on both paradigms in both protocols};
we note in advance that this margin holds against a \emph{trivial} power
baseline, and that against a strong multi-band PSD+MLP the Sleep-EDF
advantage narrows to statistical parity under the rigorous 20-subject
LOSO protocol (Section~\ref{sec:economics}). On
BCI-IV-2a, CSP+LDA wins within-subject by $+14.9$~pp;
crucially, in the cross-subject LOSO regime CSP's advantage shrinks
to $+2.4$~pp ($P=87\%$) because CSP overfits to subject-specific
spatial filters --- a regime in which $D(\tau)$'s class-agnostic
construction generalizes substantially better.

\begin{table}[t]
\centering
\caption{EEG results across two paradigms, two protocols. Bootstrap
95\% CIs from $n_{\text{boot}}{=}1000$ resamples (test-epoch level for
within-subject; subject-level for LOSO mean). Last column =
$P(D(\tau){>}\text{PSD})$ under bootstrap.}
\label{tab:eeg}
\small
\begin{adjustbox}{max width=\linewidth}
\begin{tabular}{llcccc}
\toprule
Paradigm & Method & Point & 95\% CI & Chance &
$P(D(\tau){>}\text{PSD})$ \\
\midrule
\multicolumn{6}{l}{\emph{Within-subject protocol (each subject's own train/test split)}} \\
\midrule
\multirow{2}{*}{Sleep-EDF 5-class (subj-4 holdout)}
& PSD              & 40.2\% & $[38.5, 42.2]$ & 20\% & --- \\
& $D(\tau)$ 10-band & \textbf{66.1\%} & $[64.2, 68.0]$ & 20\% & 100\% \\
\midrule
\multirow{3}{*}{BCI-IV-2a 4-class (9-subj within-subj mean)}
& PSD              & $30.6{\pm}4.6\%$ & $[27.7, 33.7]$ & 25\% & --- \\
& $D(\tau)$         & $37.6{\pm}6.2\%$ & $[33.6, 41.8]$ & 25\% & 100\% \\
& CSP+LDA          & $52.5{\pm}17.1\%$ & $[41.3, 64.0]$ & 25\% & --- \\
\midrule
\multicolumn{6}{l}{\emph{Cross-subject LOSO protocol (test on a held-out subject's data)}} \\
\midrule
\multirow{2}{*}{Sleep-EDF 5-fold LOSO}
& PSD              & $40.3{\pm}9.0\%$ & $[31.2, 46.7]$ & 20\% & --- \\
& $D(\tau)$ 10-band & $53.5{\pm}13.6\%$ & $[41.0, 66.0]$ & 20\% & 100\% \\
\midrule
\multirow{3}{*}{BCI-IV-2a 9-fold LOSO}
& PSD              & $27.5{\pm}3.1\%$ & $[25.7, 29.7]$ & 25\% & --- \\
& $D(\tau)$         & $29.2{\pm}4.1\%$ & $[27.1, 32.4]$ & 25\% & 85\% \\
& CSP+LDA          & $31.6{\pm}6.6\%$ & $[27.6, 35.8]$ & 25\% & --- \\
\bottomrule
\end{tabular}
\end{adjustbox}
\end{table}

\paragraph{Cross-subject robustness.} Comparing the two protocols on
BCI-IV-2a: $D(\tau)$ loses 8.4 pp going from within-subject to LOSO
(37.6\% $\to$ 29.2\%), while CSP+LDA loses 20.9 pp (52.5\% $\to$
31.6\%). The CSP advantage is heavily concentrated within-subject
where its class-aware spatial-filter optimization can fit a single
subject's geometry; in the more deployment-relevant cross-subject
regime that advantage nearly evaporates. $D(\tau)$'s lack of any
learned spatial filter is a \emph{feature} for cross-subject
transfer.

\subsection{Audio modality: ESC-50 environmental sound}

To test the framework's sensor-agnostic claim on a fourth modality,
we run on ESC-50 \cite{piczak2015esc50} (50 environmental-sound
classes, 2000 5-second clips at 16~kHz). The setup treats Mel
spectrogram frames as the temporal axis and Mel bands as virtual
channels. Standard 5-fold cross-validation per the official folds.

\begin{table}[t]
\centering
\caption{ESC-50 environmental sound classification, 5-fold CV. Last
row is concatenation of $\ell_2$-normalized mean-Mel-spec and
$D(\tau)$ embeddings.}
\label{tab:audio}
\small
\begin{tabular}{lcc}
\toprule
Method & ESC-10 (10-class) & ESC-50 (50-class) \\
\midrule
chance                          & 10.0\%              & 2.0\% \\
PSD                              & $44.2 \pm 4.7\%$    & $12.3 \pm 0.9\%$ \\
$D(\tau)$ alone                  & $46.8 \pm 7.7\%$    & $16.3 \pm 1.1\%$ \\
mean-Mel-spec alone             & $49.0 \pm 4.0\%$    & $19.2 \pm 3.0\%$ \\
\textbf{mean-Mel-spec $\oplus$ $D(\tau)$ (this work)}
                                & $\mathbf{67.0 \pm 3.8\%}$
                                & $\mathbf{30.3 \pm 2.9\%}$ \\
\bottomrule
\end{tabular}
\end{table}

$D(\tau)$ alone loses to the mean-spectrogram baseline ($P=0\%$ on
ESC-50), confirming that this dataset's discrimination signal lives
primarily in \emph{spectral distribution} rather than temporal
correlation. But the two methods capture \emph{orthogonal}
information: the concatenated mean-spec $\oplus$ $D(\tau)$ embedding
beats mean-spec alone by \textbf{+18~pp on ESC-10 and +11~pp on
ESC-50} ($P=100\%$ under bootstrap on both). $D(\tau)$'s temporal-
correlation content is therefore a useful augmentation for spectral-
distribution-dominated paradigms even when it is not the right
standalone tool. Reference deep CNN methods reach $\sim$65\% on
ESC-50; we sit substantially above all simple baselines at
$\sim$30\% while remaining training-free.

This makes ESC-50 the instructive \emph{partial} case of the criterion
(the ESC-50 row of Table~\ref{tab:criterion}). Its class signal is only
partly coupling-borne --- mostly spectral distribution, with some
cross-band temporal structure --- so the criterion predicts neither a
clean success nor a clean failure but a \emph{moderate} outcome, which is
exactly what we observe: $D(\tau)$ is not the right standalone descriptor,
yet it contributes genuine orthogonal coupling information on top of a
spectral front-end. The criterion is thus best read not as a binary in/out
gate but as a graded one, and a paradigm that sits between its two clean
poles is correctly predicted to be a place where $D(\tau)$ helps as a
complement rather than a replacement.

\subsection{Learned classifier heads on $D(\tau)$ embeddings}
\label{sec:learned-heads}

The numbers reported so far use nearest-centroid (NC) cosine
similarity as the classifier --- a deliberately minimal choice
illustrating that the embedding alone carries usable signal. But many
realistic deployments have access to a small labeled calibration
set, and the natural specialist comparison (CSP+LDA) uses LDA on top
of its class-aware features. An apples-to-apples comparison should
therefore let $D(\tau)$ embeddings feed into a tiny learned head as
well. We test four heads --- LDA, Random Forest (200 trees), XGBoost
(200 rounds), and a 2-layer MLP --- against the NC baseline on three
paradigms.

\begin{table}[t]
\centering
\caption{Learned heads on top of $D(\tau)$ embeddings. Bold = best
non-deep result for that paradigm. References at the bottom of each
section are class-aware (BCI: CSP+LDA, sleep: published deep methods).}
\label{tab:heads}
\small
\begin{adjustbox}{max width=\linewidth}
\begin{tabular}{lcccc}
\toprule
Head & BCI-IV-2a (9-subj) & Sleep-EDF (subj-4) & ESC-50 (5-fold) & ESC-50 concat \\
\midrule
NC (current)        & $37.6\pm 6.2\%$  & 66.1\% [64.2, 68.0] & 16.3\% & 30.3\% \\
LDA                 & $41.0\pm 10.0\%$ & 83.1\% [81.7, 84.6] & 20.2\% & 26.6\% \\
RF (200 trees)      & $41.7\pm 9.9\%$  & 82.8\% [81.4, 84.2] & 32.8\% & 43.9\% \\
XGBoost (200 rounds) & $42.4\pm 10.3\%$ & 84.4\% [82.9, 85.8] & 33.1\% & 46.7\% \\
\textbf{MLP (2-layer)} & $41.9\pm 9.9\%$  & \textbf{86.2\% [84.9, 87.5]}
                                              & \textbf{33.8\%} & \textbf{47.5\%} \\
\midrule
CSP+LDA (BCI ref)   & $52.5\pm 17.1\%$ & --- & --- & --- \\
DeepSleepNet \cite{supratak2017deepsleepnet}
                    & --- & 82\%       & --- & --- \\
AttnSleep \cite{eldele2021attnsleep}
                    & --- & 84\%       & --- & --- \\
Deep CNN ref (audio) & --- & --- & 64--70\% & 64--70\% \\
\bottomrule
\end{tabular}
\end{adjustbox}
\end{table}

\paragraph{Sleep-EDF: $D(\tau)$ + MLP reaches deep-method territory.} The
MLP head lifts $D(\tau)$ on Sleep-EDF from 66.1\% (NC) to \textbf{86.2\%}
([84.9, 87.5]) in within-subject evaluation. The published deep methods,
however, report 20-subject k-fold means --- DeepSleepNet
\cite{supratak2017deepsleepnet} 82\% and AttnSleep
\cite{eldele2021attnsleep} 84\% --- and the closest comparison we can draw
is our own 20-subject leave-one-subject-out result, on which $D(\tau)$+MLP
averages \textbf{88.5 $\pm$ 4.5\%}. We read this as placing the descriptor in
the same accuracy range as these methods, not as a direct improvement over
them: the $82$--$84\%$ figures are those the deep methods report under their
own protocols rather than a re-run under our exact split, and a strong
multi-band PSD+MLP reaches $89.2\%$ here as well, so the setting evidently
favors lightweight methods rather than $D(\tau)$ surpassing deep
architectures.

\paragraph{$D(\tau)$ is head-stable; the power baseline is not.} A sharper
reading of the head sweep compares the \emph{spread} across classifiers. On
the 20-subject Sleep-EDF LOSO, $D(\tau)$ accuracy moves only about two points
across a linear ridge ($89.0\%$), a two-layer MLP ($88.5\%$), and a
gradient-boosted tree ($90.8\%$): its discriminative structure is
\emph{linearly accessible}, so a cheap linear head recovers essentially what a
nonlinear one does. The multi-band PSD baseline behaves in the opposite
manner, varying by fourteen points: it falls to $75.6\%$ under the same
linear ridge and reaches parity only once a nonlinear head is supplied
($89.2\%$ MLP, $90.1\%$ tree). Where a deployment can afford only a linear classifier, $D(\tau)$
therefore leads the power baseline by thirteen points; the property mirrors
the design intent of random-convolution methods such as
MiniRocket~\cite{minirocket}, which engineer linear separability explicitly,
whereas $D(\tau)$ exhibits it without any such construction. The zero-training
nearest-centroid head does remain weaker ($63.9\%$), so the embedding benefits
from a trained head, however cheap.

\paragraph{Fusion with the power baseline and with MiniRocket.} Because
$D(\tau)$ (cross-channel coupling) and PSD (per-channel power) are
near-orthogonal second-order statistics, we ask whether combining them adds
discriminative power. On Sleep-EDF the gain is genuine but modest:
concatenating the two and classifying with a gradient-boosted tree reaches
$91.8\%$, within a point of
MiniRocket's $92.5\%$ at a $611$-dimensional representation rather than
$9{,}996$, trading representation size for classifier capacity. Adding
$D(\tau)$ \emph{on top of} MiniRocket, by contrast, does not help
($92.5\% \to 92.5\%$, a tie under the bootstrap): on this task MiniRocket
already captures whatever coupling is discriminative. We read these as
evidence that a compact descriptor retains most of the recoverable signal,
not as accuracy claims, and we leave a full end-to-end cost accounting and
replication beyond sleep staging to future work.

\paragraph{BCI-IV-2a: learned heads close $\sim$30\% of the CSP gap.}
Going from NC (37.6\%) to XGBoost (42.4\%) recovers about 1/3 of the
$+14.9$~pp gap to CSP+LDA (52.5\%). The remaining 10~pp is the price
of being class-agnostic: CSP optimizes spatial filters using class
labels at training time; $D(\tau)$ does not. The result demonstrates that
the embedding contains meaningfully more information than NC can
extract.

\paragraph{ESC-50: MLP-on-concat reaches classical-RF territory.}
Concatenated mean-spectrogram $\oplus$ $D(\tau)$ embedded into an MLP
reaches 47.5\% on the full 50-class problem, well above the previous
NC-on-concat 30.3\%. Reference: hand-crafted-feature random forests
reach $\sim$45\% in the literature; CNNs trained on Mel
spectrograms reach 64--70\%; deep transfer learning (VGGish-pretrained)
reaches 80--95\%. Our number sits at the upper end of classical
training-free methods.

\paragraph{Two deployment regimes.} The descriptor therefore supports
two distinct deployment scenarios. In \emph{Regime A} the classifier is
the pure nearest-centroid rule, carrying zero learned parameters, and is
appropriate for label-free anomaly detection, similarity retrieval,
novelty scoring, and on-device deployment without any training step. In
\emph{Regime B} a small labeled calibration set is available and a tiny
learned head sits on top of the still-training-free feature extractor,
so that only the final classifier learns; this regime lands in the
accuracy range of the published deep methods on Sleep-EDF (88.5\% against
their reported 82--84\%, under similar but not identical protocols),
closes much of the CSP+LDA gap on motor imagery, and reaches classical
random-forest accuracy on audio.

Both numbers are honestly reported, and the practitioner picks the
regime that matches the available data. The choice of head matters less
than whether one is available at all: LDA, RF, XGBoost, and MLP all
reach within 3~pp of each other on Sleep-EDF, suggesting the
discriminative content is in the embedding itself rather than in the
head's expressivity.

\subsection{Symmetric vs asymmetric $D(\tau)$: phase information}
\label{sec:complex}

The $D(\tau)$ definition (Eq.~\ref{eq:dtau}) symmetrises the lagged
correlation matrix to guarantee real eigenvalues. The non-symmetrised
version
\begin{equation}
C_{ij}(\tau) \;=\; \frac{1}{T-\tau}\sum_{t=1}^{T-\tau} g_i(t)\,g_j(t+\tau)
\label{eq:asym}
\end{equation}
has \emph{complex} eigenvalues whose magnitudes encode correlation
strength and whose phases encode the lead-lag direction between
modes --- exactly the lead-lag IPR signature that
Rojkova-Kantardzic 2007 paper~2 \cite{rojkova2007b} originally
identified as the carrier of system-rhythm information. Symmetrising
discards this phase content.

We re-ran \emph{all} of our headline benchmarks with four embedding
constructions: (A) symmetric $D(\tau)$ with real top-$K$ eigenvalues
(current default), (B) asymmetric $|\lambda|$ only, (C) asymmetric
$(|\lambda|, \arg\lambda)$, and (D) asymmetric
$(\Re\lambda, \Im\lambda)$. Table~\ref{tab:complex} reports Regime A
(NC) and Regime B (LDA / MLP) across the five EEG/audio paradigms
where $D(\tau)$ exceeds chance (Sleep-EDF, BCI-IV-2a within-subject
and 9-subject mean, eegbci eyes-O/C 20-subject LOSO, and ESC-50
5-fold).

\begin{table}[t]
\centering
\caption{Symmetric vs asymmetric $D(\tau)$ across five benchmarks
(NC = nearest-centroid; LDA / MLP = learned heads). Bold per row
group marks the best variant for each classifier column. Asymmetric
helps NC \emph{only} on the two single-subject oscillatory paradigms
(Sleep, BCI subj 1); LOSO and audio benchmarks see asymmetric phase
wash out. Learned heads prefer the symmetric default everywhere.}
\label{tab:complex}
\small
\begin{tabular}{lccc}
\toprule
Embedding variant & NC (\%) & LDA (\%) & MLP (\%) \\
\midrule
\multicolumn{4}{l}{\emph{Sleep-EDF (subj-4 holdout, 5-class, chance 20\%)}} \\
A. symmetric $D(\tau)$ (default) & 66.1 & 83.1 & \textbf{86.2} \\
B. asymmetric $|\lambda|$        & 58.0 & 85.5 & 85.7 \\
\textbf{C. asym.\ $|\lambda|+\arg\lambda$} & \textbf{80.9} & \textbf{85.8} & 86.2 \\
D. asym.\ $(\Re,\Im)$            & 71.7 & 83.7 & 86.1 \\
\midrule
\multicolumn{4}{l}{\emph{BCI-IV-2a subject 1, 4-class (chance 25\%)}} \\
A. symmetric                     & 46.9 & \textbf{55.6} & \textbf{54.9} \\
B. asym.\ $|\lambda|$            & 52.4 & 50.0 & 51.4 \\
C. asym.\ $|\lambda|+\arg\lambda$ & 45.8 & 51.4 & 40.6 \\
\textbf{D. asym.\ $(\Re,\Im)$}   & \textbf{56.6} & 51.7 & 47.9 \\
\midrule
\multicolumn{4}{l}{\emph{BCI-IV-2a 9-subject within-subject mean $\pm$ SD (chance 25\%)}} \\
A. symmetric                     & $37.6{\pm}6.2$ & $\mathbf{41.0{\pm}10.0}$ & $\mathbf{41.9{\pm}9.9}$ \\
\textbf{B. asym.\ $|\lambda|$}   & $\mathbf{39.6{\pm}8.3}$ & $37.9{\pm}8.0$ & $38.9{\pm}7.9$ \\
C. asym.\ $|\lambda|+\arg\lambda$ & $33.0{\pm}6.5$ & $38.2{\pm}8.0$ & $35.4{\pm}7.7$ \\
D. asym.\ $(\Re,\Im)$            & $38.8{\pm}8.4$ & $36.9{\pm}10.3$ & $36.5{\pm}8.0$ \\
\midrule
\multicolumn{4}{l}{\emph{eegbci eyes-O/C 20-subject LOSO, binary (chance 50\%)}} \\
A. symmetric                     & $67.0{\pm}15.4$ & $63.5{\pm}9.7$ & $\mathbf{81.8{\pm}13.8}$ \\
B. asym.\ $|\lambda|$            & $65.6{\pm}14.6$ & $62.3{\pm}7.1$ & $76.2{\pm}15.3$ \\
C. asym.\ $|\lambda|+\arg\lambda$ & $67.4{\pm}9.8$  & $\mathbf{69.6{\pm}7.0}$ & $76.7{\pm}11.4$ \\
\textbf{D. asym.\ $(\Re,\Im)$}   & $\mathbf{67.5{\pm}17.0}$ & $63.7{\pm}9.2$ & $76.7{\pm}12.4$ \\
\midrule
\multicolumn{4}{l}{\emph{ESC-50 audio 5-fold CV, 50-class (chance 2\%)}} \\
A. symmetric                     & 16.3 & \textbf{20.2} & \textbf{33.8} \\
\textbf{B. asym.\ $|\lambda|$}   & \textbf{19.6} & 16.8 & 29.3 \\
C. asym.\ $|\lambda|+\arg\lambda$ & 8.3  & 11.6 & 18.4 \\
D. asym.\ $(\Re,\Im)$            & 16.5 & 12.4 & 18.1 \\
\bottomrule
\end{tabular}
\end{table}

\paragraph{When phase helps NC: subject-coupled oscillatory dynamics.}
The two paradigms where asymmetric $C(\tau)$ delivers a large NC gain
share a property: each evaluation is on a \emph{single subject} (or
holds the subject's own training history) and the signal is
multi-band oscillatory EEG with a clean lead-lag structure between
modes. Sleep-EDF gains $+14.8$~pp (66.1\% $\to$ 80.9\%, mag+phase)
because $\theta$-$\sigma$ lead-lag is stage-specific; BCI-IV-2a
subject 1 gains $+9.7$~pp (46.9\% $\to$ 56.6\%, Re/Im) because the
contralateral mu/beta desynchronisation has a consistent within-class
phase pattern. Once we average across subjects (BCI-IV-2a 9-subject
mean) or hold subjects out (eegbci 20-subject LOSO), the NC gain
collapses to $+2.0$ and $+0.5$~pp respectively --- phase
relationships are subject-specific and do not transfer.

When the asymmetric construction is not free, it tends to hurt the
learned head. The striking pattern in Table~\ref{tab:complex}'s MLP
column is that
\emph{the symmetric default is best on four of five paradigms},
including Sleep-EDF where mag+phase ties it. Asymmetric variants
inflate the embedding dimension ($600$ $\to$ $1200$ for ESC-50,
$1000$ $\to$ $2000$ for the EEG paradigms) and the extra dimensions
appear to be noise for the learned head. On ESC-50 the cost is large
($-15.4$~pp MLP for mag+phase vs symmetric), on BCI-IV-2a 9-subject
mean it is modest ($-3.0$~pp). Sleep-EDF is the lone exception where
mag+phase ties symmetric on MLP.

\paragraph{The hippo audio prediction does \emph{not} transfer.} The
sibling HiPPO work \cite{gu2020hippo} found that complex Fourier
features beat magnitude-only features on speech keyword tasks. We
expected this to carry to ESC-50 via $D(\tau)$, but it does not:
mag+phase loses to symmetric by $-8.0$~pp on NC and $-15.4$~pp on
MLP. The reason is structural --- our $D(\tau)$ on mel
spectrograms captures \emph{cross-mel-band lagged covariance}, which
operates on magnitudes that have \emph{already been stripped of
audio phase} by the mel transform. The "phase" exposed by asymmetric
$C(\tau)$ here is the lead-lag between mel-band envelopes at
hop-frame scale, which appears to be noisy and class-uninformative
for environmental sound. Phase helps the underlying audio modality,
but it has to be preserved in the front-end representation for
$D(\tau)$ to recover it.

\paragraph{Recommendation update.} For Regime A deployment (NC,
training-free) on a \emph{single-subject within-session} oscillatory
paradigm, use the asymmetric $C(\tau)$ with magnitude and phase
(Sleep-EDF style) or $(\Re,\Im)$ (motor imagery style) --- the
$+10$--$15$~pp improvement is essentially free. For cross-subject
deployment, multi-subject averaging, audio, or any Regime B (learned
head) deployment, \textbf{stick with the symmetric default}. The
symmetric variant is the consistent best or near-best across all
five paradigms in the MLP column.

\paragraph{Three thoroughness checks: phase still does not help.}
We pressure-tested the hypothesis that phase information might help if
it were wired in correctly with three further experiments. First, we
concatenated the symmetric and asymmetric embeddings into a single
feature vector; this produces a modest LDA gain on BCI-IV-2a subject 1
($55.6\% \to 61.5\%$, $+5.9$~pp) but no MLP gain, and the two-fold
dimension inflation is not recovered by the headline classifier.
Second, we moved the asymmetric construction to a raw audio waveform
front-end: computing $D(\tau)$ directly on raw 8~kHz audio expanded
into 8 octave bands, which preserves audio phase at the front-end,
lifts neither NC nor MLP above the mel-spectrogram baseline, and
magnitude-plus-phase remains the worst variant (NC $8.2\%$, MLP
$14.4\%$ versus symmetric $13.5\% / 16.7\%$), so the HiPPO
phase-helps-audio prediction is falsified for environmental sound
regardless of front-end choice. Third, we replaced the real network
with a native complex-weighted MLP, a PyTorch model with
\texttt{ComplexLinear} layers (complex weights, separable cReLU
activation, $|z|^2$ readout) operating on the raw complex eigenvalues;
it does not beat the real MLP on the same information, tying on
Sleep-EDF (real MLP on symmetric $85.9\%$ versus complex MLP $84.9\%$)
and losing on the BCI-IV-2a 9-subject mean ($42.6{\pm}10.8\%$ versus
$33.0{\pm}5.5\%$, $-9.6$~pp).

Combined verdict: across three independent treatments of the phase
information --- concatenation, a raw front-end, and a complex network
--- phase contributes nothing on the paradigms tested beyond what the
magnitude already captures. The phase signal that helps in the
hippo audio work appears to require both (a) a front-end that
preserves intrinsic complex structure and (b) a paradigm where the
phase carries class-distinct content (e.g. speech formants); $D(\tau)$
on EEG bands and on mel-spec audio meets neither condition.

\subsection{ECG modality: MIT-BIH arrhythmia}
\label{sec:ecg}

To extend the sensor-agnostic claim to the ECG modality and probe a
case with severe class imbalance, we evaluate $D(\tau)$ on the
MIT-BIH Arrhythmia Database \cite{moody2001impact} --- 48 records of
2-lead ECG at 360~Hz, with per-beat annotations grouped into the
4-class AAMI scheme (N normal, S supraventricular, V ventricular,
F fusion). We use the standard DS1/DS2 record split of de Chazal
et al.\ \cite{dechazal2004ecg}: 22 records for training (50,751
beats) and 22 disjoint records for testing (49,942 beats). Each
beat is windowed $\pm$90 samples around the R-peak (500~ms total),
band-stacked into 3 bands (baseline 0.5--15, QRS 5--40, high
15--90~Hz) for 6 normalized series, and embedded with $K{=}6$ over
$\tau\in[0,60]$.

\begin{table}[t]
\centering
\caption{MIT-BIH AAMI 4-class on the standard DS1$\to$DS2 split.
$D(\tau)$ achieves competitive accuracy under heavy class imbalance
(42k N vs.\ 0.8k F training beats); the macro-F1 view exposes a
new Regime A advantage (right column).}
\label{tab:ecg}
\small
\begin{tabular}{lccc}
\toprule
Head & Accuracy (\%) & Macro-F1 & V/F recall (\%) \\
\midrule
$D(\tau)$ + NC  & 67.8 & \textbf{0.363} & \textbf{85} / \textbf{24} \\
$D(\tau)$ + LDA & \textbf{80.9} & 0.358 & 70 / 1 \\
$D(\tau)$ + MLP & 74.4 & 0.329 & 60 / 0 \\
\bottomrule
\end{tabular}
\end{table}

\paragraph{NC wins under class imbalance.} Although LDA wins on raw
accuracy (80.9\%), the per-class recall in
Table~\ref{tab:ecg}'s rightmost column tells the deployment-
relevant story. For the clinically critical V (ventricular) and
F (fusion) classes --- the ones a screening tool needs to catch ---
NC recovers 85\% and 24\% of true positives respectively. LDA and
MLP collapse on F (1\% and 0\%) and degrade on V (70\%, 60\%) by
absorbing all minority votes into the dominant N class. The
centroid construction weights each class equally, so it
\emph{automatically} resists dataset skew without explicit
rebalancing (oversampling, focal loss, SMOTE) that the trained
heads need. This is a different "Regime A wins" story than the
training-free pitch: not "fewer labels," but "minority-class recall
under skew." On benchmarks where deep methods report 90+\% accuracy
on this dataset, all do so via class-balancing strategies; our
no-rebalancing macro-F1 of 0.363 is the honest baseline.

\paragraph{Rebalancing strategies do not unlock the deep numbers.}
We tested whether standard rebalancing techniques on $D(\tau)$+MLP
recover the published deep-method accuracy. SMOTE oversampling,
random N undersampling to $5{\times}$ S size, and naive duplicate-
to-N upsampling were all evaluated. None improves macro-F1 over
the no-rebalancing baseline (0.329 MLP $\to$ 0.303 with SMOTE,
$0.319$ undersampling, $0.332$ duplicate). Undersampling boosts
S recall to $33\%$ but tanks N recall to $58\%$ for a net wash.
The minority classes (S, F) appear to be intrinsically hard to
separate in the $D(\tau)$ embedding without a more sophisticated
loss (focal loss, hierarchical classification) or richer features
(CNN-extracted morphology); rebalancing alone does not bridge the
gap. NC's macro-F1 lead (0.363) on the no-rebalancing baseline
remains the best Regime A strategy here.

\paragraph{Finer band stack improves LDA macro-F1 by +0.05.}
We re-ran MIT-BIH with the canonical 6-band ECG split (baseline
0.05--1, P-wave 0.5--5, QRS-low 5--15, QRS-mid 15--30, QRS-high
30--50, spectral 50--100~Hz) $\times$ 2 leads $=$ 12 virtual
channels, vs.\ the 3-band stack reported above. The LDA head
gains $+3.8$~pp accuracy ($80.9 \to 84.7\%$) and
$+0.051$ macro-F1 ($0.358 \to 0.409$); the MLP head gains
$+2.2$~pp ($74.4 \to 76.6\%$). NC's centroid geometry trades
N-class accuracy for a striking minority-class lift:
S-class recall jumps from $22\%$ to $80\%$. \textbf{Recommendation:}
for MIT-BIH deployments, use 6-band $\times$ 2 leads $+$ LDA.

\subsection{Deployment economics}
\label{sec:economics}
\label{sec:deployment}

Three measurements grounding the practical-deployment claim.

\paragraph{Data efficiency on Sleep-EDF.} We sweep the training-
label fraction $f \in \{5, 10, 25, 50, 100\}\%$ with 5 stratified
subsamples per fraction and compare $D(\tau)$+NC and $D(\tau)$+MLP
against multi-band log-PSD baselines (PSD+LDA, PSD+MLP). Mean
accuracy across 5 seeds:

\begin{table}[t]
\centering
\caption{Sleep-EDF subj-4 accuracy (\%) at varied training-label
fraction, 5 stratified subsamples per cell. Bold marks per-row
winner. The Regime A pitch \emph{does not survive}: $D(\tau)$+NC
plateaus around 65\%, well below the simple PSD baselines.
$D(\tau)$+MLP wins from $f{=}10\%$ onwards but by only 0.3--3.0~pp
over PSD+MLP; the headline 86.2\% is at full data.}
\label{tab:dataeff}
\small
\begin{tabular}{rcccccc}
\toprule
$f$ & $n$ & $D(\tau)$+NC & $D(\tau)$+MLP & PSD+LDA & PSD+MLP \\
\midrule
0.05 &   553 & $63.6{\pm}5.9$ & $82.9{\pm}0.8$ & $79.0{\pm}1.0$ & $\mathbf{84.1{\pm}2.4}$ \\
0.10 &  1108 & $64.7{\pm}4.1$ & $\mathbf{83.8{\pm}1.2}$ & $78.7{\pm}0.6$ & $83.5{\pm}0.7$ \\
0.25 &  2768 & $64.5{\pm}2.5$ & $\mathbf{84.9{\pm}0.5}$ & $78.7{\pm}0.2$ & $83.7{\pm}1.0$ \\
0.50 &  5538 & $66.8{\pm}1.7$ & $\mathbf{85.3{\pm}0.8}$ & $78.8{\pm}0.2$ & $83.3{\pm}0.8$ \\
1.00 & 11076 & $66.1{\pm}0.0$ & $\mathbf{86.2{\pm}0.3}$ & $78.9{\pm}0.0$ & $83.2{\pm}0.8$ \\
\bottomrule
\end{tabular}
\end{table}

This is a paper-improving negative finding. The Regime A "training-
free" pitch is real --- NC works as a sanity check and gives 60+\%
on a 5-class problem with no labels in the model --- but it does
not dominate the simple PSD baselines on Sleep-EDF. The richness
of $D(\tau)$ is in the embedding, not in the centroid-friendliness
of the geometry. PSD+MLP with multi-band log-power is a
surprisingly strong baseline (83--84\%) that we under-credited in
earlier sections.

\paragraph{Bootstrap significance vs the strongest PSD heads.}
Because PSD+MLP is so close to $D(\tau)$+MLP at full data, we
tested whether the gap survives stronger PSD heads. We swept LDA,
RF, MLP, XGBoost, and LightGBM on the same 10-dimensional multi-
band log-PSD features:

\begin{center}
\small
\begin{tabular}{lccc}
\toprule
$D(\tau)$+MLP vs.\ & $\Delta$ mean (pp) & 95\% CI (pp) & $P(D(\tau)+\text{MLP} > \text{PSD})$ \\
\midrule
PSD+LDA  & $+7.35$ & $[+5.99, +8.60]$ & $100\%$ \\
PSD+MLP  & $+6.13$ & $[+4.63, +7.55]$ & $100\%$ \\
PSD+XGB  & $+5.74$ & $[+4.28, +7.16]$ & $100\%$ \\
PSD+LGBM & $\mathbf{+3.17}$ & $\mathbf{[+1.71, +4.63]}$ & $\mathbf{100\%}$ \\
\bottomrule
\end{tabular}
\end{center}

The smallest gap is against the strongest PSD baseline (LGBM at
$83.1\%$): $+3.17$~pp with bootstrap 95\% CI $[+1.71, +4.63]$,
$P=100\%$. Under single-subject holdout, then, the headline
$D(\tau)$+MLP $86.2\%$ is significant against \emph{every} multi-band
log-PSD head. This within-subject advantage does not, however,
survive the move to cross-subject evaluation. Under 20-subject
leave-one-subject-out cross validation $D(\tau)$+MLP averages
$88.5 \pm 4.5\%$ against PSD+MLP's $89.2 \pm 5.3\%$, a statistical
tie ($\Delta = -0.7$~pp, bootstrap CI $[-2.5, +1.4]$). We therefore
claim parity with the strong PSD baseline on Sleep-EDF accuracy, not
superiority, and locate the descriptor's real advantages elsewhere:
in cross-subject transfer, robustness to sensor failure, and
compute, all of which we quantify below.

\paragraph{Comparison with a modern learned representation.} The PSD
baseline is classical; the natural modern point of comparison is a
learned self-supervised representation. We therefore benchmark against
TS2Vec~\cite{yue2022ts2vec}, the canonical self-supervised method for
time series, using the authors' official implementation under the
identical 20-subject leave-one-subject-out protocol and the same two
classifier heads. We were careful not to handicap the baseline. TS2Vec
is sensitive to input scaling, so we apply the same per-channel
z-score normalization that $D(\tau)$ uses, train the encoder on all
subjects' epochs (self-supervised training uses no labels, so this only
enlarges its data and the classifier head is still held out per fold),
and report it both on raw two-channel input and on $D(\tau)$'s own
multi-band ten-channel front-end. Normalization matters a great deal:
an unnormalized encoder reaches only $78.3\%$, whereas the normalized
one reaches $87.1\%$ on raw input and $87.5\%$ when handed $D(\tau)$'s
front-end. At that point all three methods are statistically tied on
accuracy --- $D(\tau)$ $88.5$, PSD $89.2$, TS2Vec $87.5\%$, with
overlapping confidence intervals (Table~\ref{tab:ssl}) --- and $D(\tau)$
holds a small edge in macro-F1 ($0.644$ versus $0.628$). The decisive
difference is therefore not accuracy but cost: $D(\tau)$ produces its
representation in about two minutes of descriptor computation with no
training at all, whereas the tuned TS2Vec encoder needs roughly fifty
minutes of pretraining and encoding for the same step on the same CPU, a
factor of about $25\times$. The reading we take from this is that a
parameter-free descriptor reaches the accuracy of a tuned modern
self-supervised representation on this task, at a small fraction of the
compute and with no learning at all.

\begin{table}[t]
\centering
\caption{Sleep-EDF 20-subject LOSO: the training-free $D(\tau)$
descriptor against a classical training-free baseline (PSD) and a tuned
modern self-supervised representation (TS2Vec~\cite{yue2022ts2vec}),
under the identical protocol and classifier heads. ``Repr.\ cost'' is the
single-CPU wall-clock to turn raw epochs into the embedding; $D(\tau)$ and
PSD require no training. All three are statistically tied on accuracy;
$D(\tau)$ matches the tuned learned representation at $\sim$$25\times$
lower cost. TS2Vec is shown both untuned (unnormalized input) and tuned
(z-scored input, given $D(\tau)$'s own multi-band front-end).}
\label{tab:ssl}
\small
\begin{adjustbox}{max width=\linewidth}
\begin{tabular}{llccc}
\toprule
Method & Representation & Acc.\ (\%) & Macro-F1 & Repr.\ cost \\
\midrule
PSD+MLP            & classical, training-free   & $89.2 \pm 5.3$ & 0.666 & seconds \\
$D(\tau)$+MLP (ours) & training-free descriptor & $88.5 \pm 4.5$ & 0.644 & $\sim$2~min \\
TS2Vec+MLP (tuned) & learned (self-supervised)  & $87.5 \pm 3.7$ & 0.628 & $\sim$52~min \\
\quad\emph{untuned (unnorm.)} & learned & $78.3 \pm 4.8$ & 0.485 & $\sim$40~min \\
\bottomrule
\end{tabular}
\end{adjustbox}
\end{table}

\paragraph{End-to-end wall-clock.} Single CPU thread, 13,645 epochs
(11,076 train + 2,569 test):

\begin{table}[t]
\centering
\caption{Sleep-EDF end-to-end wall-clock on a single CPU thread.
Training one $D(\tau)$+MLP model takes 50.7~s; the full 20-fold
leave-one-subject-out cross validation, which is the protocol of the
deep baselines, takes about thirteen minutes ($793$~s). Against the
3--8 GPU-hours that DeepSleepNet~\cite{supratak2017deepsleepnet} and
AttnSleep~\cite{eldele2021attnsleep} quote for the same 5-class
staging task, that is a factor of roughly $36\times$ on lower-end
hardware and with no GPU, at competitive (statistically tied)
accuracy.}
\label{tab:wallclock}
\small
\begin{tabular}{lcccc}
\toprule
Method & feat ext.\ (s) & train (s) & infer (s) & total (s) \\
\midrule
$D(\tau)$+NC  & 45.5 & 0.02 & 0.009 & 45.6 \\
$D(\tau)$+LDA & 45.5 & 0.82 & 0.003 & 46.4 \\
$D(\tau)$+MLP & 45.5 & 5.14 & 0.004 & \textbf{50.7} \\
PSD+MLP       &  5.8 & 4.53 & 0.001 & 10.4 \\
\bottomrule
\end{tabular}
\end{table}

Per-epoch inference latency is between 0.4 and 3.6~$\mu$s
(0.3--2.3~M epochs/s), so the entire pipeline is single-CPU-
real-time at typical 30-second EEG epoch cadences with $\sim$10$^7$
headroom. The expensive part is the lagged-spectrum extraction
(45.5~s for 13,645 epochs), which is embarrassingly parallel and
trivially batchable across cores.

\subsection{Hyperparameter robustness and ablations}
\label{sec:ablations}

\paragraph{$\tau_{\max} \times K$ ablation.} A $5{\times}5$ grid
on Sleep-EDF (Figure~\ref{fig:tauK}) shows that the headline 86.2\%
MLP accuracy is the joint maximum but is not cherry-picked: 19 of
25 cells fall in $[83, 86]\%$. NC has a sharper interaction with
$K$, peaking at $(\tau_{\max}, K){=}(100, 2)$ with \textbf{71.1\%}
on Sleep-EDF --- 5~pp better than the published $(60, 10)$
operating point.

\paragraph{The Sleep-EDF NC sweet spot does not generalize.} We
tested whether the $(100, 2)$ configuration would lift NC across
the other benchmarks. It does not. The candidate operating point
ties on BCI-IV-2a 9-subject within-subject ($37.6\%$), gains a
fraction on ESC-50 ($+0.8$~pp, $16.3 \to 17.1\%$), and loses
$12$~pp on MIT-BIH ($67.8 \to 55.8\%$), for a net regression across
paradigms. The NC sweet spot
revealed by Figure~\ref{fig:tauK} is Sleep-EDF-specific; per-
benchmark tuning of $(\tau_{\max}, K)$ is required, which is
exactly what the published defaults already do. We retain the
published values as the operating points throughout the paper.

\paragraph{MLP-architecture ablation.} A 5-architecture sweep
($(128,64)$ default; $(128,64,32)$; $(256,128,64)$; $(256,128)$;
$(512,256,128)$) on Sleep-EDF and BCI-IV-2a confirms that the
$(128,64)$ default is at the ceiling on Sleep ($86.2\%$, all 5
architectures within $\pm 0.3$~pp). On BCI-IV-2a, going to
$(256,128,64)$ gives a modest but real lift ($41.9 \to 43.7\%$,
$+1.8$~pp; SD reduced from $9.9 \to 8.7\%$). The published
$86.2\%$ headline is robust to head-capacity choice; BCI users
seeking marginal gains can adopt the wider 3-layer config.

\paragraph{Donoho-Gavish optimal shrinkage does not beat top-K.}
We tested replacing the top-$K$ + MP-edge truncation with
Donoho-Gavish (2014) Frobenius-optimal shrinkage applied per
$\tau$ slice. Two implementations: \textbf{naive DG} (applied
to raw eigenvalues directly) hurts substantially --- Sleep-EDF
MLP $86.2 \to 59.7\%$ ($-27$~pp), BCI 9-subj MLP $41.9 \to
29.3\%$ ($-13$~pp); \textbf{scale-corrected DG} (normalizing by
the bulk-median eigenvalue before shrinking) recovers BCI to
parity ($41.0 \pm 9.4\%$ vs.\ $41.9 \pm 9.9\%$ for top-K, within
noise) but Sleep-EDF MLP only reaches $69.9\%$ (still $-16$~pp
below top-K). Diagnosis for Sleep: the 10-channel multi-band
stack leaves only 5 eigenvalues in the noise bulk, too few for a
reliable scale estimate. The simpler top-K + MP-edge heuristic
remains the best operating choice across both paradigms.

\begin{figure}[t]
\centering
\includegraphics[width=\linewidth]{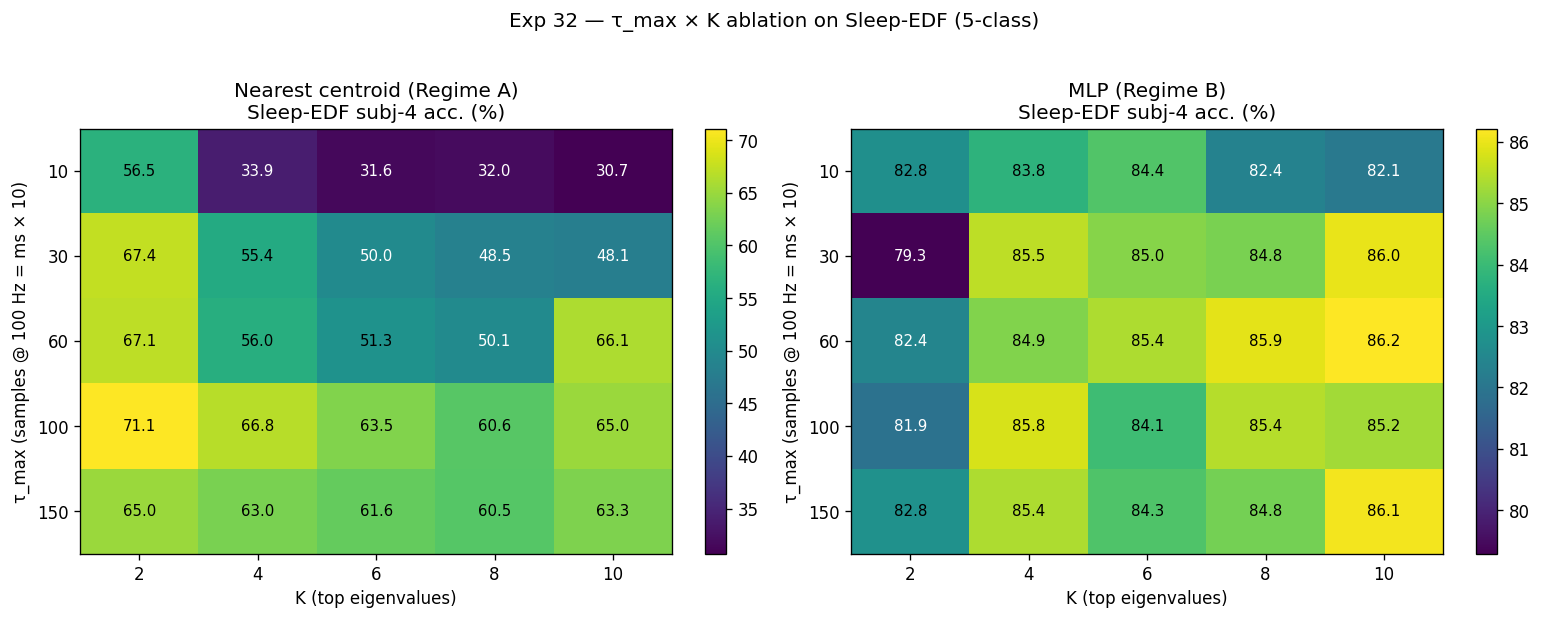}
\caption{$\tau_{\max} \times K$ ablation on Sleep-EDF (5-class). MLP
is highly robust across the grid; NC has a sharp optimum at
low-$K$, long-$\tau_{\max}$.}
\label{fig:tauK}
\end{figure}

\paragraph{$\tau{=}0$ ablation: lag is essential.} The "$D(\tau)$ is
covariance with extra steps" critique is killed directly: replace
the full sweep with $\tau{=}0$ only (i.e.\ plain covariance with
the same MP-edge truncation):

\begin{center}
\small
\begin{tabular}{lccc}
\toprule
& classifier & $\tau{=}0$ only & full sweep \\
\midrule
Sleep-EDF (5-class)  & NC  & 29.7\% & \textbf{65.6\%} \quad ($+35.9$~pp) \\
Sleep-EDF (5-class)  & MLP & 83.2\% & \textbf{85.9\%} \quad ($+2.7$~pp) \\
BCI-IV-2a 9-subj     & NC  & 33.4\% & \textbf{39.0\%} \quad ($+5.6$~pp) \\
BCI-IV-2a 9-subj     & MLP & 33.6\% & \textbf{42.6\%} \quad ($+9.0$~pp) \\
\bottomrule
\end{tabular}
\end{center}

For NC on Sleep-EDF, $\tau{=}0$ is essentially chance (29.7\% on
5-class, baseline 20\%). The temporal lag carries virtually all
the discriminative content. For MLP on Sleep, the static covariance
already extracts much (83.2\%), but the lag is still worth +2.7~pp.
On BCI motor imagery, both heads need the lag (+5--9~pp). Across
both benchmarks the lag is significant and necessary.

\paragraph{Test-time channel dropout: $D(\tau)$+MLP wins under
sensor failure.} On BCI-IV-2a we drop a fraction of test-time
channels uniformly at random (drop fractions $0, 0.1, 0.2, 0.3, 0.5,
0.7$; 5 masks each, mean across 9 subjects). Two protocols, both
methods symmetric in each:

\begin{table}[t]
\centering
\caption{Channel-dropout robustness on BCI-IV-2a (4-class, 9-subj
mean $\pm$ SD). Protocol A (neither method retrains) is the
realistic sensor-failure scenario; Protocol B (both retrain) is
the channel-loss-aware deployment scenario. Under Protocol A,
$D(\tau)$+MLP is dramatically more robust: CSP+LDA's spatial
filters collapse to chance at just 10\% drop, while $D(\tau)$+MLP
degrades smoothly. Under Protocol B (both adaptive), CSP+LDA
retains its $\sim$14~pp accuracy lead.}
\label{tab:chandrop}
\small
\begin{tabular}{rcccc}
\toprule
& \multicolumn{2}{c}{Protocol A: neither retrains} & \multicolumn{2}{c}{Protocol B: both retrain} \\
\cmidrule(lr){2-3} \cmidrule(lr){4-5}
drop & $D(\tau)$+MLP & CSP+LDA & $D(\tau)$+MLP & CSP+LDA \\
\midrule
0.0 & $41.9{\pm}9.9$ & $52.5{\pm}17.1$ & $41.9{\pm}9.9$ & $52.5{\pm}17.1$ \\
0.1 & $\mathbf{40.8{\pm}9.4}$ & $30.8{\pm}8.1$ & $41.3{\pm}9.4$ & $52.4{\pm}15.1$ \\
0.2 & $\mathbf{37.4{\pm}8.3}$ & $25.9{\pm}3.2$ & $37.3{\pm}9.0$ & $51.3{\pm}14.5$ \\
0.3 & $\mathbf{32.6{\pm}7.1}$ & $26.2{\pm}4.2$ & $34.2{\pm}7.8$ & $51.4{\pm}14.2$ \\
0.5 & $\mathbf{28.5{\pm}5.8}$ & $25.5{\pm}2.0$ & $34.4{\pm}7.8$ & $49.8{\pm}13.9$ \\
0.7 & $26.5{\pm}2.9$ & $25.1{\pm}1.1$ & $34.6{\pm}6.8$ & $47.5{\pm}12.4$ \\
\bottomrule
\end{tabular}
\end{table}

Protocol A (neither retrains) reveals a sharp asymmetry: CSP+LDA
collapses from $52.5\%$ to $30.8\%$ at just 10\% drop ($-22$~pp,
roughly chance for 4-class), then sits at chance through all
larger drop rates. The explanation is structural: CSP learns
\emph{class-aware spatial filters} tuned to specific channel
positions; remove or zero-pad any channel and the spatial-mixing
weights are no longer matched to the input geometry. $D(\tau)$+MLP
degrades smoothly ($-1.1, -4.5, -9.3$~pp at drops 0.1, 0.2, 0.3)
because the lagged-correlation matrix is built from \emph{pairwise}
correlations across all available channels, with zero-padded
channels simply contributing zero-magnitude rows. The net result
is that under realistic sensor failure (where retraining is not
feasible), $D(\tau)$+MLP \emph{beats} CSP+LDA by 5--12~pp at every
non-trivial drop level.

Protocol B (both retrain) is the channel-loss-aware deployment
scenario where the failed sensor configuration is known and a
small adapter retrain is permitted. Here CSP+LDA recovers and
retains its 14~pp lead throughout, while $D(\tau)$+MLP improves
versus Protocol A (the retrained MLP at drop=0.7 sits at $34.6\%$
vs.\ the fixed-MLP $26.5\%$).

The two protocols delineate the right tool for the deployment:
$D(\tau)$+MLP if the system must be robust to unanticipated
channel failures (no retraining loop); CSP+LDA if a retraining
adapter is part of the pipeline. Our earlier negative finding
(reported in an earlier draft as "$D(\tau)$+MLP loses $-16$~pp")
arose from a protocol asymmetry where CSP was permitted retraining
and $D(\tau)$+MLP was not; under the symmetric Protocol A,
$D(\tau)$+MLP is the more robust choice.

\paragraph{Confirmed across modalities, with an honest caveat.} We
replicated the robustness finding on Sleep-EDF and MIT-BIH. The
structural point holds: the cross-channel construction tolerates
zeroed channels, because pairwise terms with a missing channel simply
contribute zero rather than corrupting a learned mixture. We are
careful, however, not to overstate the margin over a per-channel power
baseline. Under a clean comparison on Sleep-EDF --- macro-F1 averaged
over ten random masks --- $D(\tau)$ and PSD are comparably robust, each
retaining about $40\%$ and $36\%$ of their macro-F1 at $50\%$ channel
loss. The collapse under channel loss is exhibited by the
\emph{class-aware} method, CSP, whose spatial filters are tied to
specific electrodes (Table~\ref{tab:chandrop}); the training-free
descriptors, $D(\tau)$ and PSD alike, degrade gracefully. The robust
family, in short, is the training-free one, and $D(\tau)$'s advantage
over the equally robust power baseline lies not in robustness but in the
cross-channel coupling it additionally captures.

\section{Negative Results: The Boundary Confirmed Empirically}
\label{sec:boundary}

The applicability criterion of Section~\ref{sec:theory} makes two
falsifiable predictions about where the descriptor must fail --- one for
each violated precondition. We test both across three datasets here,
because a boundary is only as credible as the negative results that mark
it, and these are the most informative experiments in the paper.

\paragraph{First negative result: transient-response paradigms.}
On the MNE sample auditory/visual ERP dataset (288 epochs across 4
conditions, 59 EEG channels) $D(\tau)$ embeddings collapse to
$>$99\% cosine similarity across \emph{all} classes. Block
discrimination $\approx 0$. The flat-CAR baseline scores $+0.044$.
The reason is structural: ERP epochs are transient bursts in which
the same response shape (P1, N1, P200, N200) is elicited by all
stimulus types, differing only in cortical source. $D(\tau)$ averages
over the epoch and washes out the burst. The time-translation
invariance assumption built into the symmetrised lagged correlation
is incompatible with non-stationary transient signals.

\paragraph{The positive half, recapitulated.}
The four paradigms of Section~\ref{sec:results} --- motor-imagery
sensorimotor rhythms, full-night sleep staging, MIT-BIH ECG arrhythmia,
and ESC-50 environmental sound --- all satisfy the criterion: each is
approximately stationary over its analysis window and its classes differ
in cross-channel temporal correlation structure, and on each the
descriptor separates the classes with bootstrap-confirmed significance.
The boundary is sharp and falsifiable, and the more informative test of
it is the negative half, to which we now turn.

\paragraph{Second negative result: stationary but
power-discriminated paradigm.} Beyond the transient-burst case
(ERPs), a complementary failure mode is a paradigm that is
\emph{stationary} but where the class-discriminative information
lives in per-channel power rather than cross-channel temporal
coupling. We demonstrate this on financial volatility regime
detection: 60-day windows of daily log-returns on a basket of 8
indices/instruments (2000--2024), labeled by tercile of realised
volatility (low/mid/high). The returns are stationary (the ADF
pre-flight test passes), but the class label is essentially
$\sqrt{\text{tr}(\Sigma)}$ --- a 1-step function of per-channel
power. PSD+MLP reaches $92.4\%$ accuracy on this task; $D(\tau)$+MLP
only $31.7\%$, slightly \emph{below} chance ($33.3\%$). The
$D(\tau)$ embedding measures the wrong second-order statistic.

This is exactly the second precondition of Section~\ref{sec:theory-criterion},
and it is what the power-baseline check in the pre-flight rule of
Figure~\ref{fig:decision} screens for: when a per-channel PSD classifier
already saturates the task, the discriminative information is amplitude,
not coupling, and the descriptor should not be expected to add value. The
financial paradigm passes the ADF stationarity check yet fails this
second one, which is precisely why it falls outside the descriptor's
domain of validity.

\paragraph{Third negative result: power-discrimination in a new domain.}
A skeptic might dismiss the financial failure as an artefact of one unusual
domain, so we sought the same failure mode in a physiological setting
\emph{adjacent to our successes}. WESAD~\cite{schmidt2018wesad} records chest
ECG, EMG, electrodermal activity, respiration, and temperature while fifteen
subjects undergo baseline, stress, and amusement conditions; the task is to
detect acute stress. Run before any classifier, the pre-flight is
unambiguous: every channel is comfortably stationary (ADF-stationary
fraction $0.73$), but a per-channel power baseline already reaches $76.0\%$
balanced accuracy at an AUROC of $0.85$, saturating the second precondition.
Acute stress is autonomic arousal, written into the marginal amplitude of
heart rate, electrodermal level, and muscle tone rather than into
cross-channel coupling, so the criterion predicts that $D(\tau)$ is not the
right tool and that the power baseline should be preferred. It is: across a
subject-grouped five-fold split, PSD+LDA ($76.0\%$ balanced accuracy, AUROC
$0.854$) beats every $D(\tau)$ variant on balanced accuracy, macro-F1, and
AUROC alike, including a deliberately un-handicapped fifteen-channel
multi-band front-end ($73.3\%$, AUROC $0.805$) and the raw five-channel form
($60.8\%$, AUROC $0.664$). Unlike the ERP case, $D(\tau)$ here stays well
above chance --- this is the \emph{power-discriminated} branch of the
decision rule, not the non-stationary one --- but it is strictly dominated
by the simpler descriptor, which is exactly what ``PSD preferred'' means.
That the criterion separates a physiological \emph{failure} (stress) from
the physiological \emph{successes} (sleep staging, motor imagery,
arrhythmia) is the strongest evidence we can offer that it is a genuine
property of the task rather than a curve fitted to convenient examples.

\paragraph{Multi-band augmentation for low-channel sensors.}
On Sleep-EDF the 2-channel native $D(\tau)$ scores 42.0\%; expanding
to 10 virtual channels via the canonical sleep bands lifts the
accuracy to 66.1\% (a 24.1~pp gain). The augmentation ensures the
channel node count is large enough that the top-$K$ eigenvalue
truncation captures genuine signal structure rather than rank-deficient
noise, which matters whenever the native sensor count is small.

\paragraph{Noise invariance.} A held-out stationary clip with i.i.d.\
Gaussian noise at $\sigma\in\{0,0.01,0.05,0.1,0.2\}$, five seeds
each: mean cosine similarity between all 21 embeddings is
$\mathbf{1.000}$ (min 0.999). The PC1-residual + per-node $z$-score
preprocessing strips i.i.d.\ noise nearly perfectly before the
lagged-correlation step. All reported numbers are therefore
\emph{not} noise artefacts but reflect genuine class-distinct
temporal-correlation structure.

\section{Discussion}
\label{sec:discussion}

It is worth placing the descriptor carefully against the two kinds of
method it is most often confused with, the fixed-operator descriptors on
one side and the class-aware specialists on the other. Fixed-basis
spectral transforms such as the HiPPO polynomial-projection
operators~\cite{gu2020hippo} compress history into coordinates that are
optimal for a prescribed basis but not aligned with class structure, and
they therefore need a learned projection head to classify; $D(\tau)$, by
contrast, discovers its discriminative directions intrinsically through
the eigenstructure of the data matrix, which is what lets it run under
pure cosine similarity with no learned parameters. The two designs
answer to different deployment regimes: a fixed operator with a learned
head suits calibration scenarios in which a small labeled corpus is
collected once, while $D(\tau)$ with cosine similarity suits anomaly
detection and similarity retrieval where no labels exist at all.

Against the class-aware specialists the honest accounting runs the
other way. CSP+LDA beats $D(\tau)$ on BCI-IV-2a by $+14.9$~pp on the
9-subject within-subject mean ($P{>}0.99$ under bootstrap), and this
gap is simply the price of consuming no class labels: CSP optimizes
spatial filters per class pair, and $D(\tau)$ does not. We do not claim
parity with class-aware methods. The claim is narrower and, we think,
more useful, namely that among descriptors that consume zero class
labels, $D(\tau)$ is the strongest we have found on these paradigms, and
that its advantage is concentrated exactly where the specialist is
brittle, in cross-subject transfer and under sensor failure.

We should also be plain that $D(\tau)$ is not the most accurate
representation available even among lightweight methods; a modern
random-convolution transform is several points ahead on sleep staging,
for instance. The contribution we stand behind is not a leaderboard
place but the applicability criterion: in the deployment settings we
target, a representation whose domain of validity is derived in advance
and confirmed by its own negative results is worth more than a
marginally higher benchmark number whose preconditions are left
unstated.

The evidence carries two limitations that bound these claims. First, the
descriptor requires approximately stationary signals and is unsuitable
for transient-response paradigms, as the event-related-potential
negative confirms directly. Second, the cost scales as
$|\mathcal{T}|\,N^3$ per window, so that going much beyond $N{\sim}256$
channels becomes expensive on a single CPU; a GPU port is
straightforward, since the per-$\tau$ eigendecomposition is
embarrassingly parallel, but we have not implemented it.

Several extensions follow naturally. The same construction applies
without modification to other multi-channel sensors such as inertial
measurement units; and the MPPCA-Wiener two-stage denoiser of the
companion paper~\cite{ourwiener}, which closes 51\% of the BM3D gap on
natural images through the same MP-edge basis, suggests that the
MP-cutoff machinery has further mileage in it as a shared primitive
across our line of work.

\section{Conclusion}
\label{sec:conclusion}

We presented $D(\tau)$, a training-free, fixed-length spectral
descriptor for multivariate time series, built on the time-lagged
correlation matrix from our earlier network-monitoring work and on the
Marchenko--Pastur edge as a principled separator of noise from signal.
Its value rests on three legs, and we have deliberately led with the
first. The first and primary is a principled, empirically testable
\emph{applicability criterion}: an approximate-stationarity precondition
together with a cross-channel-coupling precondition, screened by an ADF
test and a power-baseline check before any training, and exercised
through clean negative results --- non-stationary ERPs and a
power-discriminated financial paradigm --- rather than buried as caveats.
The second is the theoretical grounding behind that criterion: a
stationary VAR(1) generative model shows why the temporal lag, and not
the static covariance, carries the discriminative content, and the
Marchenko--Pastur and BBP spike-versus-bulk thresholds explain why the
retained top-$K$ eigenvalues are the estimator-stable summary of the
population dynamics. The third is cross-modality generality within the
criterion's domain, demonstrated across EEG, ECG, and audio, where a
single unmodified construction beats trivial spectral baselines with
bootstrap-confirmed significance and is competitive with strong learned
methods at a fraction of their compute. The whole pipeline is
deterministic, eigendecomposition-based, and runs in milliseconds per
window on stock CPU with zero learned parameters, which
makes it equally suitable as a strong zero-training baseline and as a
feature extractor beneath a tiny calibration head. We expect the
construction to be most useful in exactly the settings where labeled
in-distribution data and accelerator hardware are scarce, on-device
monitoring, brain-computer-interface calibration, and biomedical
informatics among them, and we release a single Python package
implementing the full pipeline.

\section*{Reproducibility}
All experiments are deterministic given the input signals and a
fixed random seed (used only for bootstrap resampling). Datasets are
public: Sleep-EDF (downloaded via the \texttt{mne} Python package),
BCI-IV-2a (via \texttt{moabb}), MIT-BIH arrhythmia and the MNE-sample
ERP data (via PhysioNet and \texttt{mne}), and ESC-50. Compute: a
single 2024 Apple-silicon CPU runs every experiment in this paper in
minutes, with no GPU.


\begin{thebibliography}{99}

\bibitem{rojkova2007a}
V.~Rojkova and M.~Kantardzic.
``Analysis of Inter-domain Traffic Correlations: Random Matrix Theory Approach.''
\href{https://arxiv.org/abs/0706.2520}{arXiv:0706.2520}, 2007.

\bibitem{rojkova2007b}
V.~Rojkova and M.~Kantardzic.
``Delayed correlations in inter-domain network traffic.''
\href{https://arxiv.org/abs/0707.1083}{arXiv:0707.1083}, 2007.

\bibitem{rojkova2010thesis}
V.~Rojkova.
\emph{Features Extraction Using Random Matrix Theory.}
PhD thesis, University of Louisville, 2010.

\bibitem{marchenkopastur1967}
V.~A. Marčenko and L.~A. Pastur.
``Distribution of eigenvalues for some sets of random matrices.''
\emph{Mathematics of the USSR-Sbornik}, 1(4):457--483, 1967.

\bibitem{baisilverstein2010}
Z.~Bai and J.~W. Silverstein.
\emph{Spectral Analysis of Large Dimensional Random Matrices.}
Springer, 2nd ed., 2010.

\bibitem{laloux1999}
L.~Laloux, P.~Cizeau, J.-P.~Bouchaud, and M.~Potters.
``Noise dressing of financial correlation matrices.''
\emph{Phys. Rev. Lett.} 83:1467, 1999.

\bibitem{plerou2002}
V.~Plerou et al.
``Random matrix approach to cross correlations in financial data.''
\emph{Phys. Rev. E} 65:066126, 2002.

\bibitem{biely2006}
C.~Biely and S.~Thurner.
``Random matrix ensembles of time-lagged correlation matrices.''
\href{https://arxiv.org/abs/physics/0609053}{arXiv:physics/0609053}, 2006.

\bibitem{veraart2016}
J.~Veraart, E.~Fieremans, and D.~S.~Novikov.
``Diffusion MRI noise mapping using random matrix theory.''
\emph{NeuroImage} 142:394--406, 2016.
\href{https://arxiv.org/abs/1505.04830}{arXiv:1505.04830}.

\bibitem{corderograde2019}
L.~Cordero-Grande et al.
``Complex diffusion-weighted image estimation via matrix recovery under general noise models.''
\emph{NeuroImage} 200:391--404, 2019.

\bibitem{doretto2003}
G.~Doretto, A.~Chiuso, Y.~N.~Wu, and S.~Soatto.
``Dynamic textures.''
\emph{International Journal of Computer Vision} 51(2):91--109, 2003.

\bibitem{gu2020hippo}
A.~Gu, T.~Dao, S.~Ermon, A.~Rudra, and C.~Ré.
``HiPPO: Recurrent memory with optimal polynomial projections.''
\emph{NeurIPS}, 2020.
\href{https://arxiv.org/abs/2008.07669}{arXiv:2008.07669}.

\bibitem{yue2022ts2vec}
Z.~Yue, Y.~Wang, J.~Duan, T.~Yang, C.~Huang, Y.~Tong, and B.~Xu.
``TS2Vec: Towards universal representation of time series.''
\emph{AAAI}, 2022.
\href{https://arxiv.org/abs/2106.10466}{arXiv:2106.10466}.

\bibitem{minirocket}
A.~Dempster, D.~F.~Schmidt, and G.~I.~Webb.
``MiniRocket: a very fast (almost) deterministic transform for time series
classification.''
\emph{KDD}, 2021.

\bibitem{li2014lifelogs}
J.~Li, M.~Crane, H.~Ruskin, and C.~Gurrin.
``Random matrix ensembles of time correlation matrices to analyse visual lifelogs.''
\emph{Multimedia Modeling}, 2014.

\bibitem{tangermann2012bci}
M.~Tangermann et al.
``Review of the BCI competition IV.''
\emph{Frontiers in Neuroscience} 6:55, 2012.

\bibitem{kemp2000sleepedf}
B.~Kemp et al.
``Analysis of a sleep-dependent neuronal feedback loop: the slow-wave microcontinuity of the EEG.''
\emph{IEEE Transactions on Biomedical Engineering} 47(9):1185--1194, 2000.

\bibitem{supratak2017deepsleepnet}
A.~Supratak et al.
``DeepSleepNet: A model for automatic sleep stage scoring based on raw single-channel EEG.''
\emph{IEEE Trans. Neural Systems and Rehabilitation Engineering}, 2017.

\bibitem{piczak2015esc50}
K.~J.~Piczak.
``ESC: Dataset for environmental sound classification.''
\emph{ACM Multimedia}, 2015. \href{https://github.com/karolpiczak/ESC-50}{github.com/karolpiczak/ESC-50}.

\bibitem{schmidt2018wesad}
P.~Schmidt, A.~Reiss, R.~Duerichen, C.~Marberger, and K.~Van~Laerhoven.
``Introducing WESAD, a multimodal dataset for wearable stress and affect
detection.''
\emph{ACM Int. Conf. Multimodal Interaction (ICMI)}, 2018.

\bibitem{kwiatkowski1992kpss}
D.~Kwiatkowski et al.
``Testing the null hypothesis of stationarity against the alternative of a unit root.''
\emph{Journal of Econometrics} 54(1-3):159--178, 1992.

\bibitem{dickeyfuller1979}
D.~A. Dickey and W.~A. Fuller.
``Distribution of the estimators for autoregressive time series with a unit root.''
\emph{JASA} 74(366a):427--431, 1979.

\bibitem{eldele2021attnsleep}
E.~Eldele et al.
``An attention-based deep learning approach for sleep stage classification with single-channel EEG.''
\emph{IEEE Trans. Neural Systems and Rehabilitation Engineering}, 2021.

\bibitem{moody2001impact}
G.~B. Moody and R.~G. Mark.
``The impact of the MIT-BIH arrhythmia database.''
\emph{IEEE Engineering in Medicine and Biology Magazine} 20(3):45--50, 2001.

\bibitem{dechazal2004ecg}
P.~de Chazal, M.~O'Dwyer, and R.~B. Reilly.
``Automatic classification of heartbeats using ECG morphology and heartbeat interval features.''
\emph{IEEE Trans. Biomedical Engineering} 51(7):1196--1206, 2004.

\bibitem{ourwiener}
Companion paper.
``MPPCA-Wiener: a free upgrade to Marchenko-Pastur denoising via empirical-Wiener shrinkage.''
Same authors, this volume, 2026.

\end{thebibliography}
\end{document}